\documentclass{article}

\usepackage{microtype}
\usepackage{graphicx}
\usepackage{subcaption}
\usepackage{booktabs} 
\usepackage{amsfonts}       
\usepackage{nicefrac}       
\usepackage{microtype,color}      
\usepackage{aliascnt}
\usepackage{comment}
\usepackage{amsmath}
\usepackage{amsfonts}
\usepackage{amsthm}
\usepackage{caption}
\usepackage{subcaption}
\usepackage{enumitem}
\usepackage{physics}
\usepackage{todonotes}
\usepackage{mathtools}
\usepackage{soul}
\usepackage{multirow}

\newtheorem{proposition}{Proposition}[section]

\newtheorem{lemma}{Lemma}[section]
\theoremstyle{definition}

\newtheorem{assumption}{Assumption}[section]

\theoremstyle{remark}

\newcommand{\wass}{\mathcal{W}}
\newcommand{\pne}{\mathcal{P}^{\mathrm{OT},N}_\epsilon}
\newcommand{\rd}{\mathrm{d}}
\newcommand{\vertiii}[1]{{\left\vert\kern-0.15ex\left\vert\kern-0.15ex\left\vert #1 
    \right\vert\kern-0.15ex\right\vert\kern-0.15ex\right\vert}}

\makeatletter
\newcommand*{\addFileDependency}[1]{
  \typeout{(#1)}
  \@addtofilelist{#1}
  \IfFileExists{#1}{}{\typeout{No file #1.}}
}
\makeatother

\usepackage{hyperref}
\usepackage[capitalise, noabbrev]{cleveref}

\usepackage[accepted]{icml2021}

\icmltitlerunning{Differentiable Particle Filtering}

\begin{document}

\twocolumn[
\icmltitle{Differentiable Particle Filtering via Entropy-Regularized Optimal Transport}



\icmlsetsymbol{equal}{*}

\begin{icmlauthorlist}
\icmlauthor{Adrien Corenflos}{equal,aalto}
\icmlauthor{James Thornton}{equal,ox}
\icmlauthor{George Deligiannidis}{ox}
\icmlauthor{Arnaud Doucet}{ox}
\end{icmlauthorlist}

\icmlaffiliation{aalto}{Department of Electrical Engineering and Automation, Aalto University}
\icmlaffiliation{ox}{Department of Statistics, University of Oxford}

\icmlcorrespondingauthor{Adrien Corenflos}{adrien.corenflos@aalto.fi}
\icmlcorrespondingauthor{James Thornton}{james.thornton@spc.ox.ac.uk}
\icmlkeywords{Monte Carlo, Particle Filtering, State-Space Models, Optimal Transport}

\vskip 0.3in
]

\printAffiliationsAndNotice{\icmlEqualContribution} 

\begin{abstract}

Particle Filtering (PF) methods are an established class of procedures for performing inference in non-linear state-space models. Resampling is a key ingredient of PF, necessary to obtain low variance likelihood and states estimates. However, traditional resampling methods result in PF-based loss functions being non-differentiable with respect to model and PF parameters. In a variational inference context, resampling also yields high variance gradient estimates of the PF-based evidence lower bound. By leveraging optimal transport ideas, we introduce a principled differentiable particle filter and provide convergence results. We demonstrate this novel method on a variety of applications.
\end{abstract}

\section{Introduction}\label{sec:intro}
In this section we provide a brief introduction to state-space models (SSMs) and PF methods. We then illustrate one of the well-known limitations of PF \cite{kantas2015particle}: resampling steps are required in order to compute low-variance estimates, but these estimates are not differentiable w.r.t.\ to model and PF parameters. This hinders end-to-end training. We discuss recent approaches to address this problem in econometrics, statistics and machine learning (ML), outline their limitations and our contributions.  

    \subsection{State-Space Models} \label{sec:state_space_models}
            SSMs are an expressive class of sequential models, used in numerous scientific domains including econometrics, ecology, ML and robotics; see e.g. \cite{chopin2020introduction,douc2014nonlinear,doucet2018sequential,kitagawa1996smoothness,lindsten2013backward,thrun2005probabilistic}. SSM may be characterized by a latent $\mathcal{X}$-valued Markov process $(X_t)_{t\geq1}$ and  $\mathcal{Y}$-valued observations $(Y_t)_{t\geq1}$ satisfying $X_1\sim \mu_{\theta}(\cdot)$ and for $t\geq1$ 
            \begin{equation}\label{eq:SSMmmodel}
                X_{t+1}|\{ X_t=x \}\sim f_{\theta}(\cdot|x),~~Y_t|\{ X_t=x \}\sim g_{\theta}(\cdot|x),
            \end{equation}
            where $\theta \in \Theta$ is a parameter of interest. Given observations $(y_t)_{t\geq1}$ and parameter values $\theta$, one may perform state inference at time $t$ by computing the posterior of $X_t$ given $y_{1:t}\coloneqq(y_1,...,y_t)$ where
            \begin{align*}
                p_{\theta}(x_{t}|y_{1:t-1})&=\int f_{\theta}(x_t|x_{t-1}) p_{\theta}(x_{t-1}|y_{1:t-1})\textrm{d}x_{t-1},\\ p_{\theta}(x_t|y_{1:t})&=\frac{g_{\theta}(y_t|x_t)p_{\theta}(x_t|y_{1:t-1})}{\int g_{\theta}(y_t|x_t)p_{\theta}(x_t|y_{1:t-1})\textrm{d}x_t},
            \end{align*}
            with $p_{\theta}(x_1|y_0):=\mu_{\theta}(x_1)$. 
            
            The log-likelihood $\ell(\theta)=\log p_{\theta}(y_{1:T})$ is then given by
            \begin{equation*}
                \ell(\theta)=\sum_{t=1}^T \log p_{\theta}(y_t|y_{1:t-1}),
            \end{equation*}
            with $ p_\theta(y_1|y_0)\coloneqq\int  g_{\theta}(y_1|x_1) \mu_{\theta}(x_1) \textrm{d}x_1$ and for $t\geq 2$
            \begin{align*}
             p_{\theta}(y_t|y_{1:t-1})&=\int g_{\theta}(y_t|x_t) p_{\theta}(x_t|y_{1:t-1})\textrm{d}x_t.
            \end{align*}
        The posteriors $p_{\theta}(x_t|y_{1:t})$ and log-likelihood $p_\theta(y_{1:T})$ are available analytically for only a very restricted class of SSM such as linear Gaussian models. For non-linear SSM, PF provides approximations of such quantities.
    
    \subsection{Particle Filtering} \label{sec:particle_filtering}
        PF are Monte Carlo methods entailing the propagation of $N$ weighted particles $(w^i_t,X^i_t)_{i\in[N]}$, here $[N]\coloneqq\{ 1,...,N \}$, over time to approximate the filtering distributions $p_{\theta}(x_t|y_{1:t})$ and log-likelihood $\ell(\theta)$. Here $X^i_t\in\mathcal{X}$ denotes the value of the $i^{\text{th}}$ particle at time $t$ and $\mathbf{w}_{t}\coloneqq(w_{t}^{1},...,w_{t}^{N})$ are weights satisfying $w^i_t\geq 0,~\sum_{i=1}^N w^i_t=1$. Unlike variational methods, PF methods provide consistent approximations under weak assumptions as $N \rightarrow \infty$ \cite{delmoral2004}. 
        In the general setting, particles are sampled according to proposal distributions $q_{\phi}(x_{1}|y_1)$ at time $t=1$ and  $q_{\phi}(x_{t}|x_{t-1},y_t)$ at time $t\geq2$ prior to weighting and resampling. One often chooses $\theta=\phi$ but this is not necessarily the case \cite{le2017auto,maddison2017filtering,naesseth2017variational}.

        
        \begin{algorithm}
        \caption{Standard Particle Filter}
        \label{alg:pf}
        \setlength{\parindent}{0pt}
        \begin{algorithmic}[1]
            \STATE{Sample $X_{1}^{i}\stackrel{\text{i.i.d.}}{\sim} q_{\phi}(\cdot|y_1)$ ~for $i \in [N]$}
           \STATE{Compute $\omega_1^{i}=\frac{p_\theta(X^{i}_1,y_1)}{q_{\phi}(X^{i}_1|y_1)}$~for $i \in [N]$}
          \STATE{$\hat{\ell}(\theta) \leftarrow\frac{1}{N}\sum_{i=1}^{N}\omega_1^{i}$}
        \FOR{$t=2,...,T$}
            \STATE{Normalize weights $w^i_{t-1}\propto \omega^i_{t-1}$, $\sum_{i=1}^N  w^i_{t-1}=1$}
            \STATE{Resample $\tilde{X}^i_{t-1}\sim \sum_{i=1}^N  w^i_{t-1} \delta_{X^i_{t-1}}$ ~for $i \in [N]$}
                \STATE{Sample $X_{t}^{i}\sim q_{\phi}(\cdot|\tilde{X}^i_{t-1},y_t)$  ~for $i \in [N]$}
                \STATE{Compute   $\omega^i_t=\frac{p_\theta(X_{t}^{i},y_t|\tilde{X}^i_{t-1})}{q_{\phi}(X_{t}^{i}|\tilde{X}^i_{t-1},y_t)}$~for $i \in [N]$}
            \STATE{Compute $\hat{p}_{\theta}(y_{t}|y_{1:t-1})=\frac{1}{N}\sum_{i=1}^{N}\omega^i_t $}
            \STATE{$\hat{\ell}(\theta) \leftarrow \hat{\ell}(\theta)+ \log{\hat{p}}_{\theta}(y_{t}|y_{1:t-1})$}
        \ENDFOR
        \STATE{{\bfseries Return:} log-likelihood estimate $\hat{\ell}(\theta)=\log \hat{p}_{\theta}(y_{1:T})$}
        \end{algorithmic}
        \end{algorithm}
        \vspace{-0.3cm}
         A generic PF is described in \cref{alg:pf} where $p_\theta(x_1,y_1)\coloneqq\mu_\theta(x_1)g_\theta(y_1|x_1)$ and $p_\theta(x_t,y_t|x_{t-1})\coloneqq f_\theta(x_t|x_{t-1})g_\theta(y_t|x_t)$.  Resampling is performed in step 6 of \cref{alg:pf}; it ensures particles with high weights are replicated and those with low weights are discarded, allowing one to focus computational efforts on `promising' regions.
        The scheme used in \cref{alg:pf} is known as multinomial resampling and is unbiased (as are other traditional schemes such as stratified and systematic \cite{chopin2020introduction}), i.e.
        \begin{equation}\label{resamplingunbiased}
        \mathbb{E}\left[\tfrac{1}{N} {\textstyle\sum}_{i=1}^N \psi(\tilde{X}^i_t)\right]=\mathbb{E}\left[ {\textstyle\sum}_{i=1}^N w_t^i \psi(X_t^i)\right],
        \end{equation}
       for any $\psi:\mathcal{X}\rightarrow \mathbb{R}$. This property guarantees $\exp \small(\hat{\ell}(\theta)\small)$ is an unbiased estimate of the likelihood $\exp(\ell(\theta))$ for any $N$.
        
        Henceforth, let $\mathcal{X}=\mathbb{R}^{d_x}$, $\theta\in \Theta=\mathbb{R}^{d_{\theta}}$ and $\phi\in \Phi=\mathbb{R}^{d_{\phi}}$. We assume here that $\theta \mapsto \mu_{\theta}(x)$, $\theta \mapsto f_{\theta}(x'|x)$ and $\theta \mapsto g_{\theta}(y_t|x)$ are differentiable for all $x,x'$ and $t\in[T]$ and $\theta \mapsto {\ell}(\theta)$ is differentiable. These assumptions are satisfied by a large class of SSMs. We also assume that we can use the reparameterization trick \cite{kingma2013auto} to sample the particles; i.e.  we have $\Gamma_\phi(y_1,U) \sim q_\phi(x_1|y_1), \Psi_\phi(y_t,x_{t-1},U) \sim q_{\phi}(x_t|x_{t-1},y_t)$ for some mappings $\Gamma_\phi,\Psi_\phi$ differentiable w.r.t. $\phi$ and $U\sim \lambda$, $\lambda$ being independent of $\phi$.

    \subsection{Related Work and Contributions}\label{sec:simulated_likelihood} \label{sec:related}
    Let  $\bf{U}$ be the set of all random variables used to sample and resample the particles. The distribution of $\bf{U}$ is $(\theta,\phi)$-independent as we use the reparameterization trick\footnote{For example, multinomial resampling relies on $N$ uniform random variables.}. However, even if we sample and fix $\bf{U}=\bf{u}$, resampling involves sampling from an atomic distribution and introduces discontinuities in the particles selected when $\theta,\phi$ vary. 

    \begin{figure}[!ht]
    \centering
        \begin{subfigure}{\linewidth}
            \centering
            \includegraphics[width=0.45\linewidth]{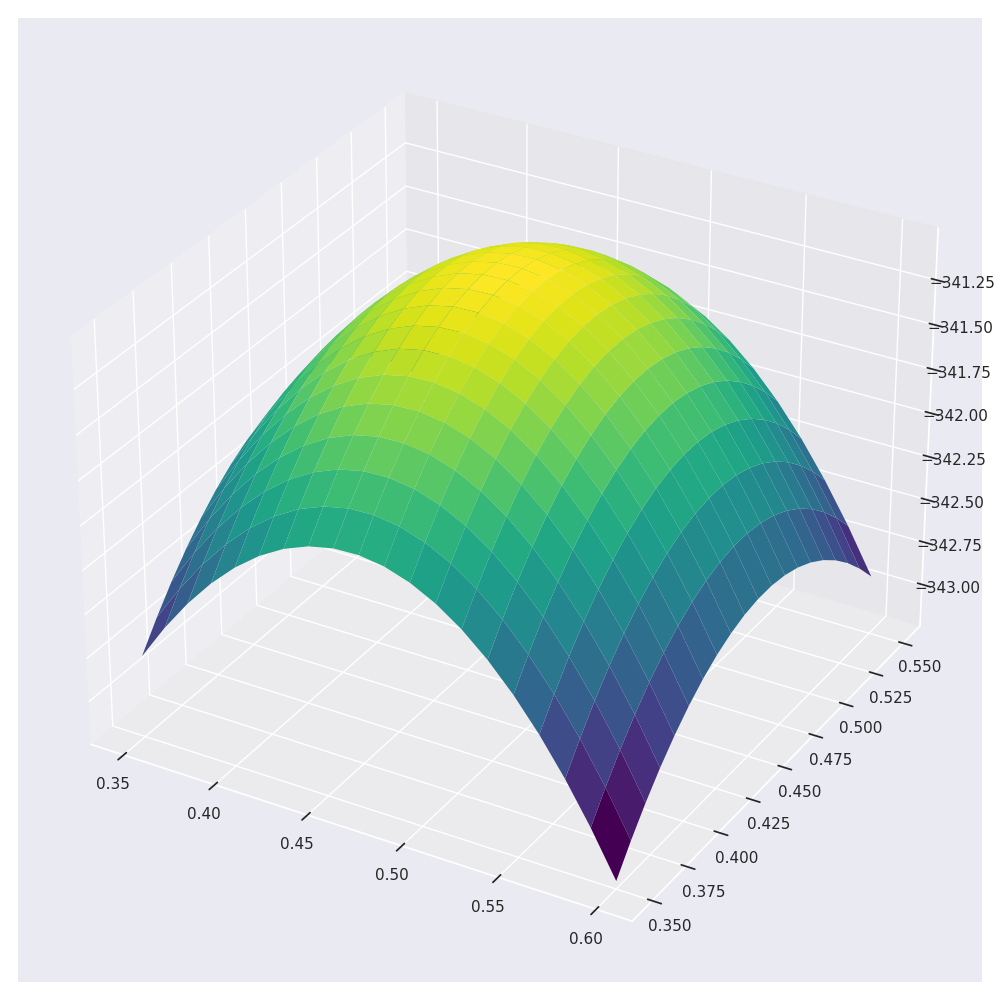}
            \includegraphics[width=0.45\linewidth]{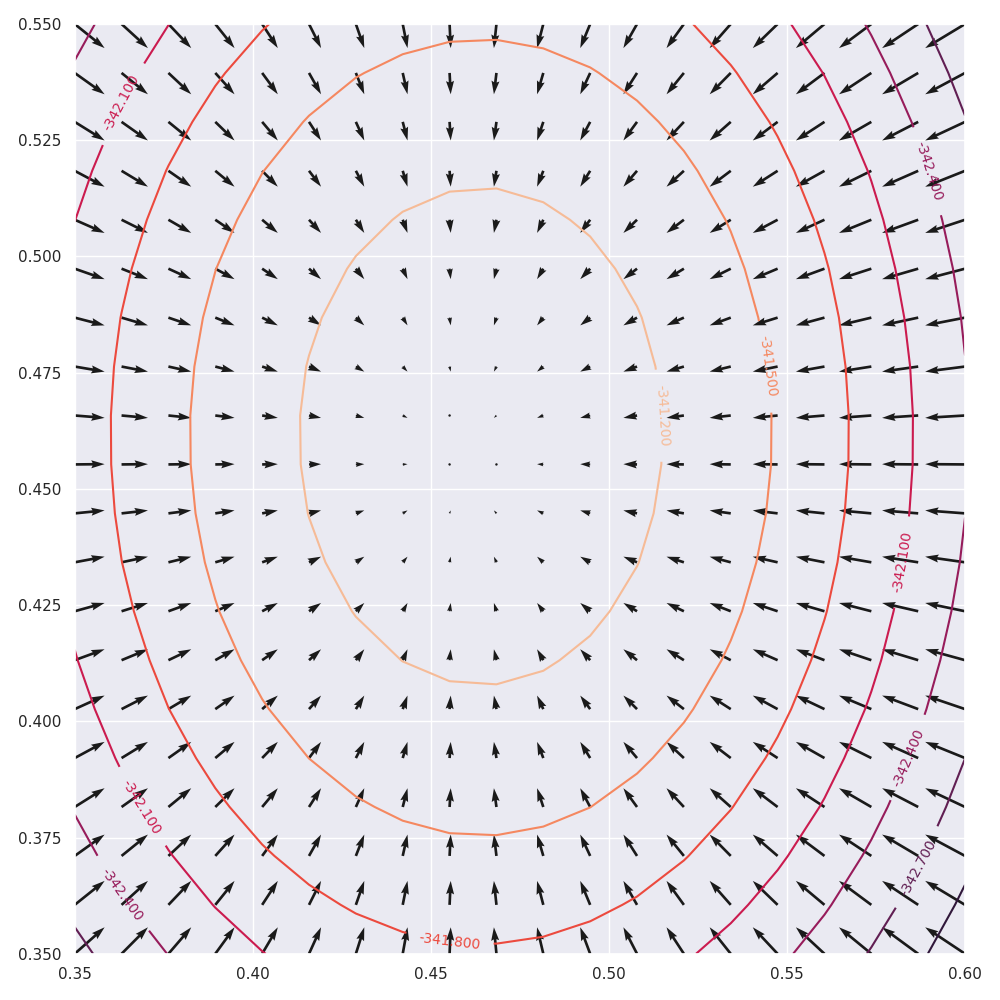}
            \caption{Kalman Filter}
        \end{subfigure}
        \begin{subfigure}{\linewidth}
            \centering
            \includegraphics[width=0.45\linewidth]{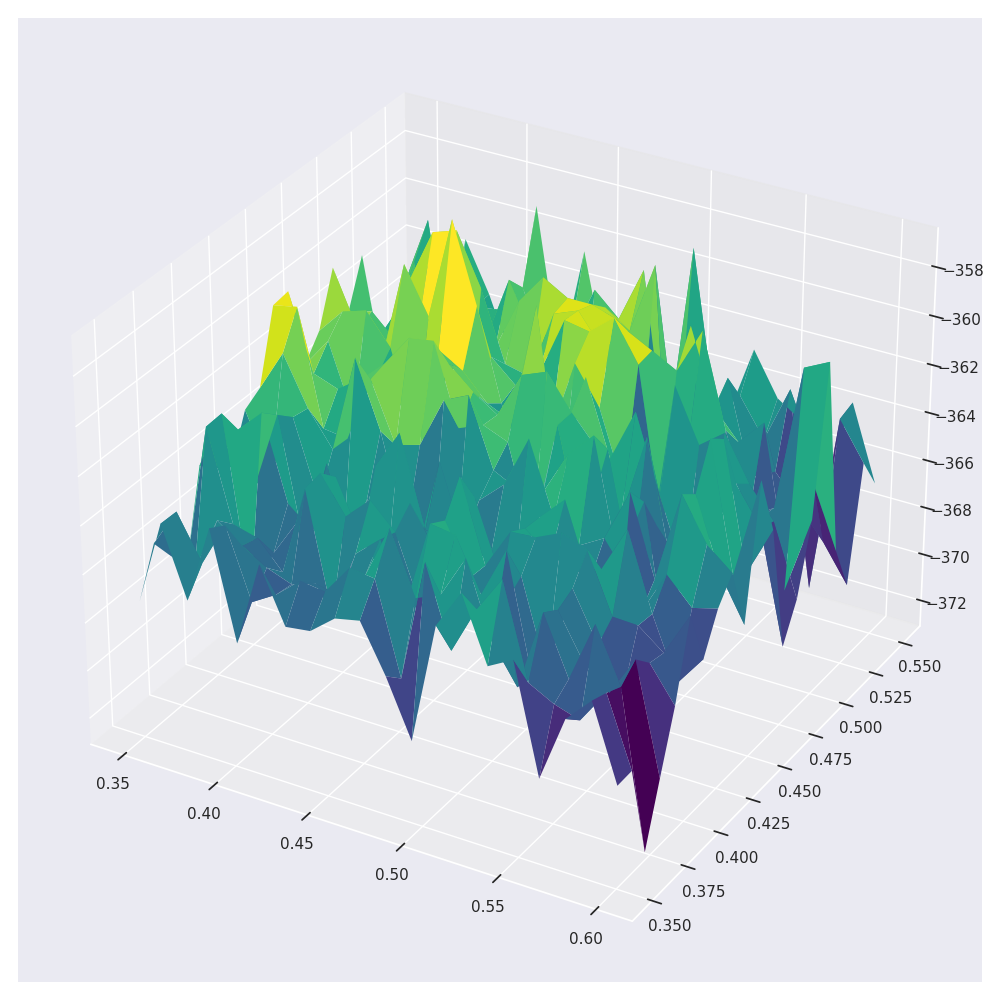}
            \includegraphics[width=0.45\linewidth]{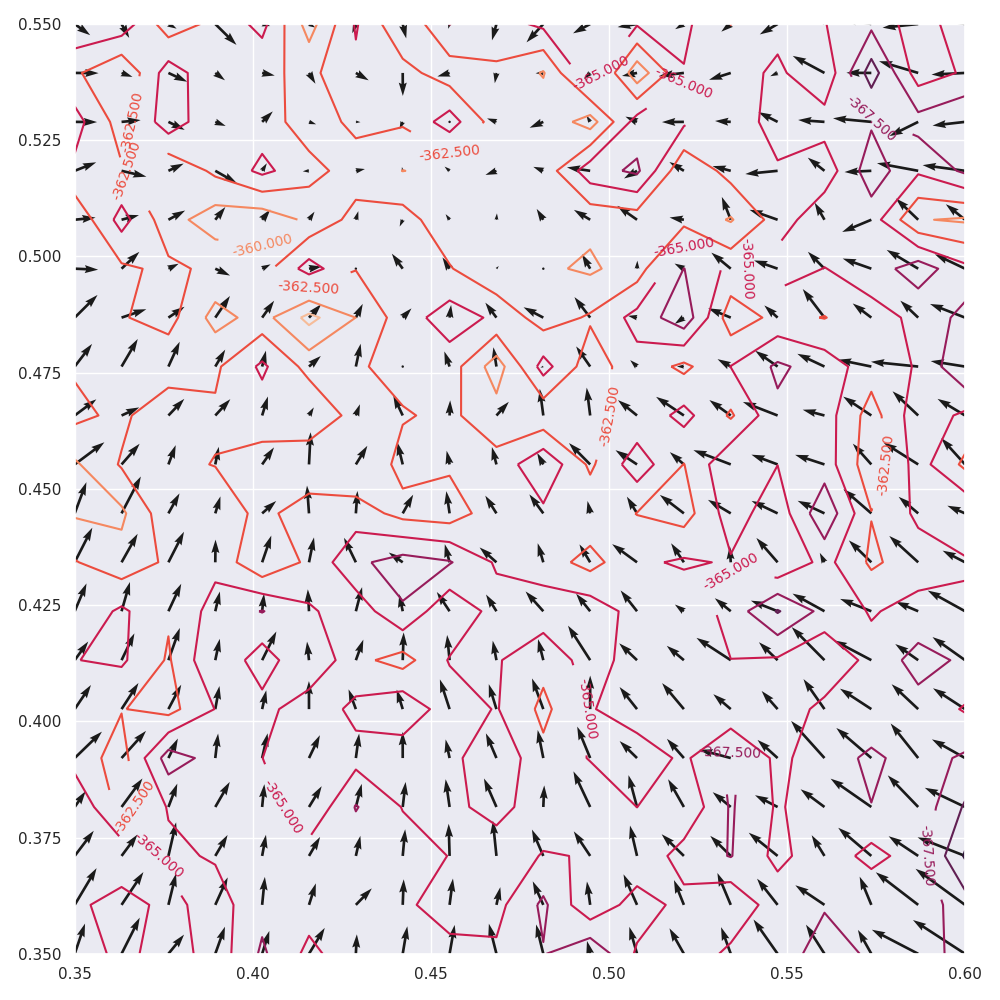}
            \caption{Standard PF}
        \end{subfigure}
        \begin{subfigure}{\linewidth}
            \centering
            \includegraphics[width=0.45\linewidth]{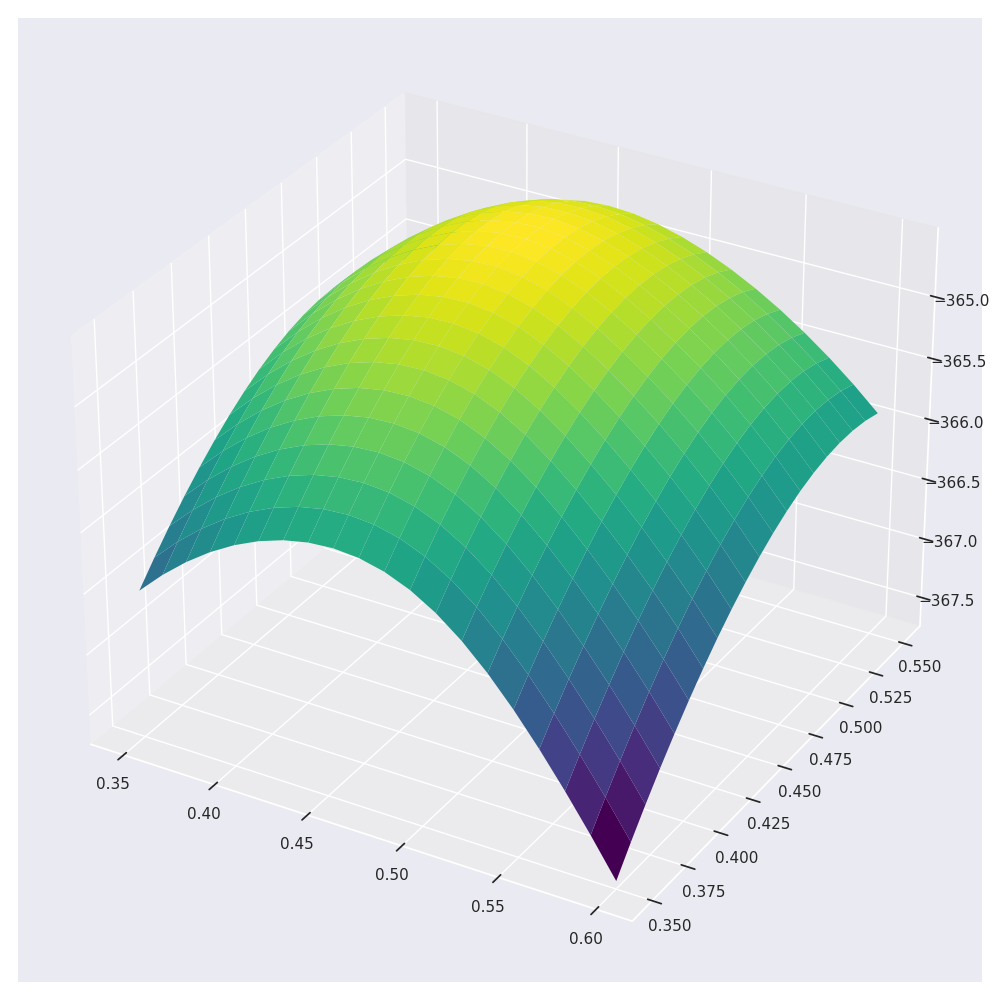}
            \includegraphics[width=0.45\linewidth]{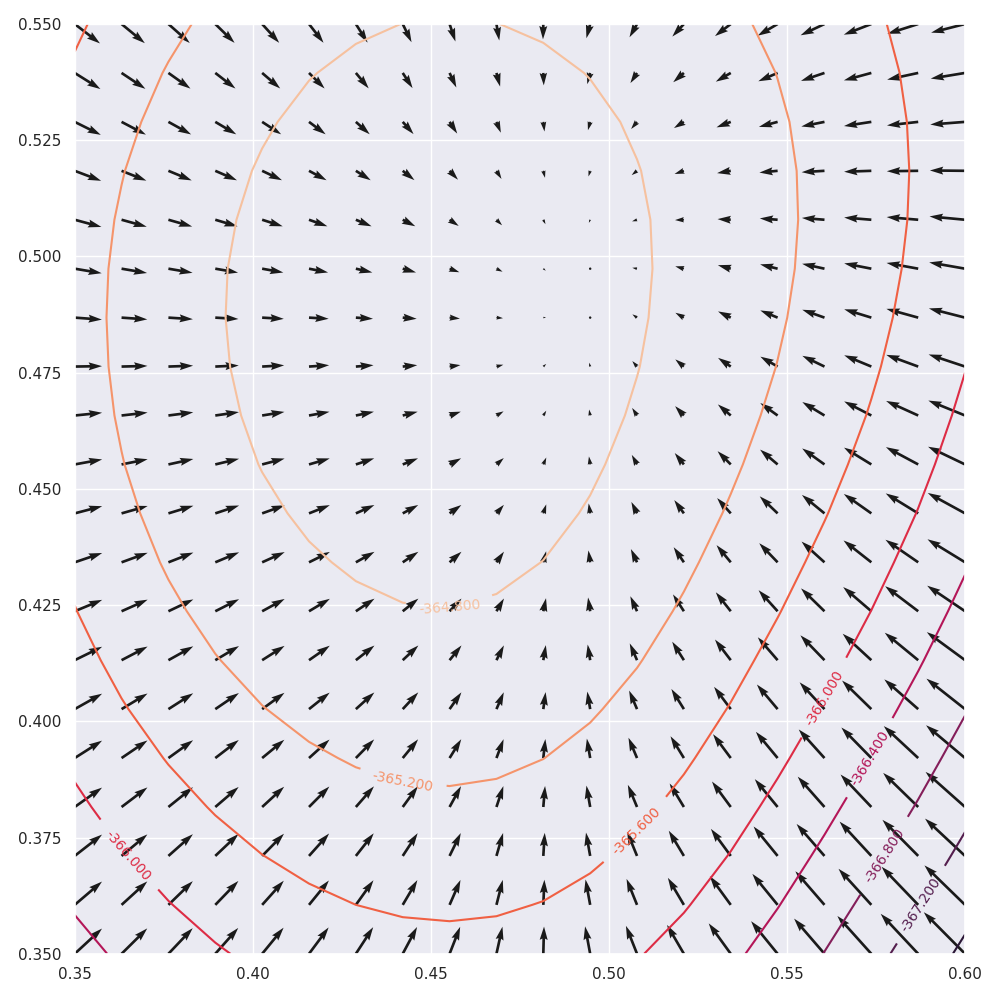}
            \caption{Differentiable PF}
        \end{subfigure}

        \caption{Left: Log-likelihood $\ell(\theta)$ and PF estimates $\hat{\ell}(\theta;\phi,\bf{u})$ for linear Gaussian SSM, given in Section \ref{subsec:LGM}, with $d_\theta=2$ $d_x=2$, and $T=150,N=50$.
        Right: $\nabla_{\theta} \ell(\theta)$ and $\nabla_{\theta} \hat{\ell}(\theta;\phi,\bf{u})$.}
        \label{fig:surfaces}
        \label{fig:vectorField}
        \vspace{-0.5cm}
    \end{figure}

     For $d_x=1$, \citet{malikpitt2011particle} make $\theta \mapsto \hat{\ell}(\theta;\phi,\mathbf{u})$ continuous w.r.t. $\theta$ by sorting the particles and then sampling from a smooth approximation of their cumulative distribution function. For $d_x>1$, \citet{lee2008towards} proposes a smoother but only piecewise continuous estimate. \citet{dejong2013efficient} returns a differentiable log-likelihood estimate $\hat{\ell}(\theta; \phi, \mathbf{u})$ by using a marginal PF \cite{klaas2012toward}, where importance sampling is performed on a collapsed state-space. However, the standard marginal PF uses the proposal $q_\phi(x_t):=\sum_{i=1}^N w_{t-1}^{i}q_\phi(x_t|X^i_{t-1},y_t)$ from which one cannot generally sample smoothly for arbitrary mixture components. As a consequence they instead suggest using a simple Gaussian distribution for $q_\phi(x_t)$,  which can lead to poor estimates for multimodal posteriors. Moreover, in contrast to standard PF, this marginal PF cannot be applied in scenarios where the transition density can only be sampled from (e.g. using the reparameterization trick) but not evaluated pointwise \cite{murray2013disturbance}, as the importance weight would be intractable. The implicit reparameterization method of \citet{graves2016stochastic} may be used to obtain low variance gradient estimates with a mixture proposal. This method however is only compatible for component distributions with tractable conditional CDFs, such as Gaussian distributions. An alternative unbiased estimate of the likelihood based on dynamic programming may also be obtained  \cite{finke2016embedded,aitchison2018tensor}. As emphasized by \citet{aitchison2018tensor}, this estimate is differentiable. This approach is again limited however to a restricted class of proposal distributions, such as an unweighted mixture proposal, which may perform poorly for slow-mixing time-series.
     
     In the context of robot localization, a modified resampling scheme has been proposed in \cite{karkus2018particle,ma2019particle,ma2020discriminative} referred to as `soft-resampling' (SPF). SPF has parameter $\alpha\in[0,1]$ where $\alpha=1$ corresponds to regular PF resampling and $\alpha=0$ is essentially sampling particles uniformly at random. The resulting PF-net is said to be differentiable but computes gradients that ignore the non-differentiable component of the resampling step. \citet{jonschkowski2018differentiable} proposed another PF scheme which is said to be differentiable but simply ignores the non-differentiable resampling terms and proposes new states based on the observation and some neural network. This approach however does not propagate gradients through time. Finally, \citet{zhu2020towards} propose a differentiable resampling scheme based on transformers but they report that the best results are achieved when not backpropagating through it, due to exploding gradients. Hence no fully differentiable PF is currently available in the literature \cite{kloss2020train}.

    PF methods have also been fruitfully exploited in Variational Inference (VI) to estimate $\theta,\phi$ \cite{le2017auto,maddison2017filtering,naesseth2017variational}. As $\mathbb{E}_{\mathbf{U}}[\exp \big(\hat{\ell}(\theta;\phi,\mathbf{U})\big)]=\exp(\ell(\theta))$ is an unbiased estimate of $\exp(\ell(\theta))$ for any $N, \phi$ for standard PF, then one has indeed by Jensen's inequality
        \begin{equation}\label{ELBO}
        \ell^{\mathrm{ELBO}}(\theta,\phi)\coloneqq\mathbb{E}_{\mathbf{U}}[\hat{\ell}(\theta;\phi,\mathbf{U})]\leq \ell(\theta).
        \end{equation}
    The standard ELBO corresponds to $N=1$ and many variational families for approximating $p_\theta(x_{1:T}|y_{1:T})$ have been proposed in this context \cite{archer2015black,krishnan2017structured,rangapuram2018deep}. The variational family induced by a PF differs significantly as $ \ell^{\mathrm{ELBO}}(\theta,\phi) \rightarrow \ell(\theta)$ as $N\rightarrow \infty$ and thus yields a variational approximation converging to $p_\theta(x_{1:T}|y_{1:T})$. This attractive property comes at a computational cost; i.e. the PF approach trades off fidelity to the posterior with computational complexity. 
    While unbiased gradient estimates of the PF-ELBO \eqref{ELBO} can be computed, they suffer from high variance as the resampling steps require having to use REINFORCE gradient estimates \cite{williams1992simple}. Consequently, \citet{HirtDellaportas2019,le2017auto,maddison2017filtering,naesseth2017variational} use biased gradient estimates which ignore these terms, yet report improvements as $N$ increases over standard VI approaches and Importance Weighted Auto-Encoders (IWAE) \cite{burda2016importance}. 
    
    Finally, if one is only interested in estimating $\theta$ (and not some distinct $\phi$), then particle techniques approximating pointwise the score vector $\nabla_\theta \ell(\theta)$ are also available \cite{poyiadjis2011particle,kantas2015particle}.
        
The contributions of this paper are four-fold. 

\begin{itemize}[topsep=0pt,itemsep=0pt,partopsep=1pt,parsep=1pt,listparindent=2pt,leftmargin=12pt]
    \item We propose the first fully Differentiable Particle Filter (DPF) which can use general proposal distributions. DPF provides a differentiable estimate of  $\ell(\theta)$, see Figure \ref{fig:surfaces}-c, and more generally differentiable estimates of PF-based losses. Empirically, in a VI context, DPF-ELBO gradient estimates also exhibit much smaller variance than those of PF-ELBO.
    
    \item We provide quantitative convergence results on the differentiable resampling scheme and establish consistency results for DPF.
    
    \item We show that existing techniques provide inconsistent gradient estimates and that the non-vanishing  bias can be very significant, leading practically to unreliable parameter estimates.

    \item We demonstrate that DPF empirically outperforms recent alternatives for end-to-end parameter estimation on a variety of applications.
\end{itemize}
Proofs of results are given in the Supplementary Material.

\section{Resampling via Optimal Transport} \label{sec:resampling_via_ot}

    \subsection{Optimal Transport and the Wasserstein Metric}\label{sec:optimal_transport}
    Since Optimal Transport (OT) \cite{peyr2019computational,villani2008optimal} is a core component of our scheme, the basics are presented here. Given two probability measures $\alpha,\beta$ on $\mathcal{X}=\mathbb{R}^{d_x}$ the squared 2-Wasserstein metric between these measures is given by       
    \begin{align}\label{eq:Wasserstein}
            \wass_2^{2}(\alpha, \beta) &=\min_{\mathcal{P}\in \mathcal{U}(\alpha,\beta)} \mathbb{E}_{(U,V) \sim \mathcal{P}}\big[||U-V||^2\big],
        \end{align}
        where $\mathcal{U}(\alpha,\beta)$ the set of distributions on $\mathcal{X}\times\mathcal{X}$ with marginals $\alpha$ and $\beta$, and the minimizing argument of (\ref{eq:Wasserstein}) is the OT plan denoted $\mathcal{P}^{\mathrm{OT}}$. Any element $\mathcal{P} \in \mathcal{U}(\alpha,\beta)$ allows one to ``transport'' $\alpha$ to $\beta$ (and vice-versa) i.e.
        \begin{align*}
            &\beta(\textrm{d}v) =\int \mathcal{P}(\textrm{d}u,\textrm{d}v) = \int \mathcal{P}(\textrm{d}v|u)\alpha(\textrm{d}u).
        \end{align*}
For atomic probability measures $\alpha_N = \sum_{i=1}^N a_i \delta_{u_i}$ and $\beta_N = \sum_{j=1}^N b_j \delta_{v_j}$ with weights $\mathbf{a}=(a_i)_{i \in [N]}$, $\mathbf{b}=(b_j)_{j \in [N]}$, and atoms $\mathbf{u}=(u_i)_{i \in [N]}$, $\mathbf{v}=(v_j)_{j \in [N]}$, one can show that    
    \begin{align}\label{eq:Wassersteinatomic}
             \wass_2^{2}(\alpha_N, \beta_N)=\min_{\mathbf{P}\in \mathcal{S}(\mathbf{a}, \mathbf{b})} \textstyle{\sum}_{i=1}^N\sum_{j=1}^N c_{i,j}p_{i,j},
        \end{align}
where any $\mathcal{P}\in\mathcal{U}(\alpha_N,\beta_N)$ is of the form 
\begin{equation*}
    \mathcal{P}(\textrm{d}u,\textrm{d}v)={\textstyle\sum}_{i,j}~p_{i,j}\delta_{u_i}(\textrm{d}u)\delta_{v_j}(\textrm{d}v),
\end{equation*}
$c_{i,j}=||u_i-v_j||^2$, $\mathbf{P}=(p_{i,j})_{i,j \in[N]}$ and $\mathcal{S}(\mathbf{a}, \mathbf{b})=\big\{\mathbf{P} \in [0,1]^{N \times N}:\sum_{j=1}^Np_{i,j}=a_i,~\sum_{i=1}^N p_{i,j} = b_j\big\}.$
In such cases, one has
        \begin{equation}
            \mathcal{P}(\textrm{d}v|u=u_i)={\textstyle\sum}_{j}~a^{-1}_i
             p_{i,j}\delta_{v_j}(\textrm{d}v).\label{eq:transportmeasures_discrete}
        \end{equation}
        The optimization problem \eqref{eq:Wassersteinatomic} may be solved through linear programming.
        It is also possible to exploit the dual formulation
        \begin{equation}\label{eq:Wassersteindual}
            \wass_2^{2}(\alpha_N, \beta_N) =\max_{\mathbf{f},\mathbf{g} \in \mathcal{R}(C)} \mathbf{a}^{t}\mathbf{f}+\mathbf{b}^{t}\mathbf{g},
        \end{equation}
        where $\mathbf{f}=(f_{i})$, $\mathbf{g}=(g_{i})$, $\mathbf{C}=(c_{i,j})$ and $\mathcal{R}(\mathbf{C})= \{\mathbf{f},\mathbf{g} \in \mathbb{R}^{N} |  f_i+g_j \leq c_{i,j}, i,j\in[N]\}$.

    \subsection{Ensemble Transform Resampling}
        The use of OT for resampling in PF has been pioneered by \citet{reich2012nonparametric}. Unlike standard resampling schemes \cite{chopin2020introduction,doucet2018sequential}, it relies not only on the particle weights but also on their locations. 
        
        At time $t$, after the sampling step (Step 7 in \cref{alg:pf}), $\alpha^{(t)}_N=\tfrac{1}{N}\sum_{i=1}^N \delta_{X_t^i}$ is a particle approximation of $\alpha^{(t)}\coloneqq\int q_\phi(x_t|x_{t-1},y_t) p_{\theta}(x_{t-1}|y_{1:t-1})\textrm{d}x_{t-1}$ and $\beta^{(t)}_N=\sum w_t^{i} \delta_{X_t^i}$ is an approximation of $\beta^{(t)}\coloneqq p_\theta(x_t|y_{1:t})$. Under mild regularity conditions, the OT plan minimizing $\wass_2(\alpha^{(t)}, \beta^{(t)})$ is of the form $\mathcal{P}^{\mathrm{OT}}(\textrm{d}x,\textrm{d}x')=\alpha^{(t)}(\textrm{d}x)\delta_{\mathbf{T}^{(t)}(x)}(\textrm{d}x')$ where $\mathbf{T}^{(t)}:\mathcal{X} \rightarrow \mathcal{X}$ is a deterministic map; i.e if $X\sim \alpha^{(t)}$ then $\mathbf{T}^{(t)}(X)\sim \beta^{(t)}$. It is shown in \cite{reich2012nonparametric} that one can one approximate this transport map with the `Ensemble Transform' (ET) denoted $\mathbf{T}^{(t)}_{N}$. This is found by solving the OT problem (\ref{eq:Wassersteinatomic}) between $\alpha^{(t)}_N$ and $\beta^{(t)}_N$ and taking an expectation w.r.t. \eqref{eq:transportmeasures_discrete}, that is
        \begin{equation}\label{eq:Reichaveraging}
        \tilde{X}^i_t=N {\textstyle\sum}_{k=1}^N~p^{\mathrm{OT}}_{i,k} X^k_t\coloneqq\mathbf{T}^{(t)}_{N}(X^i_t),
        \end{equation}
        where we slightly abuse notation as $\mathbf{T}^{(t)}_{N}$ is a function of $X^{1:N}_t$. 
        \citet{reich2012nonparametric} uses this update instead of using $\tilde{X}^i_{t}\sim \sum_{i=1}^N  w^i_{t} \delta_{X^i_{t}}$. This is justified by the fact that, as $N\rightarrow \infty$, $\mathbf{T}^{(t)}_{N}(X^i_t) \rightarrow \mathbf{T}^{(t)}(X^i_t)$ in some weak sense \cite{reich2012nonparametric,myers2019sequential}. Compared to standard resampling schemes, the ET only satisfies \eqref{resamplingunbiased} for affine functions $\psi$.
        
       This OT approach to resampling involves solving the linear program \eqref{eq:Wasserstein} at cost $O(N^3\log N)$ \cite{bertsimas1997introduction}. This is not only prohibitively expensive but moreover the resulting ET is not differentiable. To address these problems, one may instead rely on entropy-regularized OT \cite{cuturi2013sinkhorn}.

\section{Differentiable Resampling via Entropy-Regularized Optimal Transport}
    \label{sec:differentiable}
    \subsection{Entropy-Regularized Optimal Transport} \label{sec:diff_ot}
        Entropy-regularized OT may be used to compute a transport matrix that is differentiable with respect to inputs and computationally cheaper than the non-regularized version, i.e. we consider the following regularized version of (\ref{eq:Wassersteinatomic}) for some $\epsilon>0$ \cite{cuturi2013sinkhorn, peyr2019computational}
        \begin{equation}\label{eq:Wassersteinentropy}
            \hspace{-0.25cm}\wass^2_{2, \epsilon}(\alpha_N, \beta_N) =\hspace{-0.2cm}\min_{\mathbf{P}\in \mathcal{S}(\mathbf{a}, \mathbf{b})} \sum_{i,j=1}^N  p_{i,j} \Big(c_{i,j}+\epsilon \log\frac{p_{i,j}}{a_i b_j}\Big).\hspace{-0.15cm}
        \end{equation}
        The function minimized in (\ref{eq:Wassersteinentropy}) is strictly convex and hence admits a unique minimizing argument $\mathbf{P}^{\mathrm{OT}}_{\epsilon}=(p^{\mathrm{OT}}_{\epsilon,i,j})$. $\wass^2_{2, \epsilon}(\alpha_N, \beta_N)$ can also be computed using the regularized dual; i.e. $\wass^2_{2, \epsilon}(\alpha_N, \beta_N) =\max_{\mathbf{f},\mathbf{g}}\mathrm{DOT}_{\epsilon}(\mathbf{f},\mathbf{g})$ with
        \begin{align}\label{eq:Wassersteinentropydualrange}
        &\mathrm{DOT}_{\epsilon}(\mathbf{f},\mathbf{g})\coloneqq\mathbf{a}^{t}\mathbf{f}+\mathbf{b}^{t}\mathbf{g}-\epsilon\mathbf{a}^{t}\mathbf{M}\mathbf{b}
        \end{align}
        where $(\mathbf{M})_{i,j}\coloneqq\exp (\epsilon^{-1} (f_i+g_j-c_{i,j}))-1$ and $\mathbf{f},\mathbf{g}$ are now unconstrained. For the dual pair $(\mathbf{f}^{*},\mathbf{g}^{*})$ maximizing (\ref{eq:Wassersteinentropydualrange}),
        we have $\nabla_{\mathbf{f}, \mathbf{g}}\mathrm{DOT}_\epsilon(\mathbf{f}, \mathbf{g})|_{(\mathbf{f^*}, \mathbf{g^*})} = \mathbf{0}$. 
        This first-order condition leads to
        \begin{equation}\label{eq:sinkhorn}
        f_i^{*} = \mathcal{T}_\epsilon(\mathbf{b}, \mathbf{g}^*, \mathbf{C}_{i:}), \hspace{1cm}
        g_i^{*} = \mathcal{T}_\epsilon(\mathbf{a}, \mathbf{f}^*, \mathbf{C}_{:i}),
        \end{equation}
        where $\mathbf{C}_{i:}$ (resp. $\mathbf{C}_{:i}$) is the $i^{\text{th}}$ row (resp. column) of $\mathbf{C}$.
        Here $\mathcal{T}_\epsilon: \mathbb{R}^N \times \mathbb{R}^N \times \mathbb{R}^{N} \rightarrow \mathbb{R}^N$ denotes the mapping 
        \begin{equation}\notag
        \mathcal{T}_\epsilon(\mathbf{a}, \mathbf{f}, \mathbf{C}_{:,i}) = -\epsilon~\log \sum_k \exp\big\{\log a_k +  \epsilon^{-1} \left(f_k - c_{k,i}\right)\big\}.
        \end{equation}
        One may then recover the regularized transport matrix as
        \begin{equation}\label{eq:tranport_plan}
          p^{\mathrm{OT}}_{\epsilon,i,j}=a_ib_j\exp\left(\epsilon^{-1} (f^{*}_i+g^{*}_j-c_{i,j})\right).
        \end{equation}
        
        The dual can be maximized using the Sinkhorn algorithm introduced for OT in the seminal paper of \citet{cuturi2013sinkhorn}. \cref{algo:sinkhorn} presents the implementation of \citet{feydy2018interpolating} where the fixed point updates based on \cref{eq:sinkhorn} have been stabilized.
    
            \begin{algorithm}[H]
            \caption{Sinkhorn Algorithm}
            \label{algo:sinkhorn}
            \begin{algorithmic}[1]
                \STATE{{\bfseries Function} $\mathbf{Potentials}(\mathbf{a},\mathbf{b},\mathbf{u},\mathbf{v})$}
                \STATE{\bfseries Local variables:} $\mathbf{f}, \mathbf{g} \in \mathbb{R}^N$
                \STATE{\bfseries Initialize:} $\mathbf{f} = \mathbf{0}$, $\mathbf{g} = \mathbf{0}$
                \STATE{Set $\mathbf{C} \leftarrow \mathbf{u} \mathbf{u}^t+ \mathbf{v} \mathbf{v}^t-2 \mathbf{u} \mathbf{v}^t$}
                \WHILE{stopping criterion not met}
                      \FOR{$i\in[N]$}
                            \STATE $f_i \leftarrow \frac 1 2 \left(f_i +\mathcal{T}_\epsilon(\mathbf{b}, \mathbf{g}, \mathbf{C}_{i:}) \right)$
                            \STATE $g_i \leftarrow \frac 1 2 \left(g_i+\mathcal{T}_\epsilon(\mathbf{a}, \mathbf{f}, \mathbf{C}_{:i}) \right)$
                        \ENDFOR
                \ENDWHILE
                \STATE{{\bfseries Return} {$\mathbf{f}, \mathbf{g}$ }}
            \end{algorithmic}
            \end{algorithm}
           \vspace{-0.3cm}
            The resulting dual vectors $(\mathbf{f}^{*},\mathbf{g}^{*})$ can then be differentiated for example using automatic differentiation through the Sinkhorn algorithm loop \cite{flamary2018wasserstein}, or more efficiently using ``gradient stitching'' on the dual vectors at convergence, which we do here (see \citet{feydy2018interpolating} for details). The derivatives of $\mathbf{P}^{\mathrm{OT}}_{\epsilon}$ are readily accessible by combining the derivatives of (\ref{eq:sinkhorn}) with the derivatives of (\ref{eq:tranport_plan}), using automatic differentiation at no additional cost.

    \subsection{Differentiable Ensemble Transform Resampling} \label{sec:differentiable_et}
We obtain a differentiable ET (DET), denoted $\mathbf{T}^{(t)}_{N, \epsilon}$, by computing the entropy-regularized OT using Algorithm \ref{algo:differentiableResampling} for the weighted particles ($\mathbf{X}_t, \mathbf{w}_t, N$) at time $t$
    \begin{align}\label{eq:Reichaveragingregularized}
                \tilde{X}^i_t=N {\textstyle\sum}_{k=1}^N~p^{\mathrm{OT}}_{\epsilon,i,k} X^k_t\coloneqq\mathbf{T}^{(t)}_{N, \epsilon}(X^i_t).
    \end{align}
    
    \begin{algorithm}[!ht]
    		\caption{DET Resampling}
    		\label{algo:differentiableResampling}
    		\begin{algorithmic}[1]
    		    \STATE{ {\bfseries Function} $\mathbf{EnsembleTransform}(\mathbf{X}, \mathbf{w}, N$)}
    			\STATE{$\mathbf{f}, \mathbf{g} \leftarrow$ $\mathbf{Potentials}(\mathbf{w},\frac{1}{N}\mathbf{1},\mathbf{X},\mathbf{X})$}
    			
    			\FOR{$i\in[N]$}
    			    \FOR{$j \in [N]$}
    			        \STATE {$p^{\mathrm{OT}}_{\epsilon,i,j} =\frac{ w_i}{N}\exp\Big(\frac{f_i+g_j-c_{i,j}}{\epsilon}\Big)$}
    			     \ENDFOR
    			 \ENDFOR
    			 \STATE{{\bfseries Return} $\tilde{\mathbf{X}}=N\mathbf{P}^{\mathrm{OT}}_{\epsilon} \mathbf{X}$}
    		\end{algorithmic}
    \end{algorithm}
   \vspace{-0.3cm}
Compared to the ET, the DET is differentiable and can be computed at cost $O(N^2)$ as it relies on the Sinkhorn algorithm. This algorithm converges quickly \cite{altschuler2017near} and is particularly amenable to GPU implementation.

The DPF proposed in this paper is similar to \cref{alg:pf} except that we sample from the proposal $q_\phi$ using the reparameterization trick and Step 6 is replaced by the DET. While such a differentiable approximation of the ET has previously been suggested in ML \cite{cuturidoucet2014,seguy2017large}, it has never been realized before that this could be exploited to obtain a DPF. In particular, we obtain differentiable estimates of expectations w.r.t.  the filtering distributions with respect to $\theta$ and $\phi$ and, for a fixed ``seed'' $\mathbf{U}=\mathbf{u}$~\footnote{Here $\mathbf{U}$ denotes only the set of $\theta,\phi$-independent random variables used to generate particles as, contrary to standard PF, DET resampling does not rely on any additional random variable.}, we obtain a differentiable estimate of the log-likelihood function $\theta \mapsto \hat{\ell}_{\epsilon}(\theta; \phi, \mathbf{u})$.

	    
Like ET, DET only satisfies \eqref{resamplingunbiased} for affine functions $\psi$. Unlike $\mathbf{P}^{\mathrm{OT}}$, $\mathbf{P}^{\mathrm{OT}}_{\epsilon}$ is sensitive to the scale of $\mathbf{X}_t$. To mitigate this sensitivity, one may compute $\delta(\mathbf{X}_t) = \sqrt{d_x} \max_{k\in[d_x]} \text{std}_{i}(X^i_{t,k})$ for $\mathbf{X}_t \in \mathbb{R}^{N \times d_x}$ and rescale $\mathbf{C}$ accordingly to ensure that $\epsilon$ is approximately independent of the scale and dimension of the problem. 

\section{Theoretical Analysis}
We show here that the gradient estimates of PF-based losses ignoring gradients terms due to resampling are not consistent and can suffer from a large non-vanishing bias. On the contrary, we establish that DPF provides consistent and differentiable estimates of the filtering distributions and log-likelihood function. This is achieved by obtaining novel quantitative convergence results for the DET.

\subsection{Gradient Bias from Ignoring Resampling Terms} 
We first provide theoretical results on the asymptotic bias of the gradient estimates computed from PF-losses, by dropping the gradient terms from resampling, as adopted in  \cite{HirtDellaportas2019,jonschkowski2018differentiable,karkus2018particle,le2017auto,ma2020discriminative,maddison2017filtering,naesseth2017variational}. We limit ourselves here to the ELBO loss. Similar analysis can be carried out for the non-differentiable resampling schemes and losses considered in robotics. 
\begin{proposition}\label{prop:asymptoticELBOgradient} Consider the PF in \cref{alg:pf} where $\phi$ is distinct from $\theta$ then, under regularity conditions, the expectation of the ELBO gradient estimate $\hat{\nabla}_{\theta}\ell^{\textup{ELBO}}(\theta,\phi)$ ignoring resampling terms considered in \cite{le2017auto,maddison2017filtering,naesseth2017variational} converges as $N\rightarrow \infty$ to 
\begin{align*}
    &\mathbb{E}[\hat{\nabla}_{\theta}\ell^{\textup{ELBO}}(\theta,\phi)] {\rightarrow}\int \nabla_{\theta}\log p_{\theta}(x_1,y_1)~p_{\theta}(x_1|y_1)\textup{d}x_1 \notag
    \\
    &+ {\textstyle\sum\limits_{t=2}^T}\int\nabla_{\theta} \log p_{\theta}(x_t,y_t|x_{t-1})~p_{\theta}(x_{t-1:t}|y_{1:t})\textup{d}x_{t-1:t}
\end{align*}
whereas Fisher's identity yields
\begin{align}\label{eq:scoreidentity}
    & \nabla_{\theta}\ell(\theta)= \int \nabla_{\theta}\log p_{\theta}(x_1,y_1)~ p_{\theta}(x_1|y_{1:T})\textup{d}x_1\\
&+ {\textstyle\sum\limits_{t=2}^T} \int \nabla_{\theta} \log p_{\theta}(x_t,y_t|x_{t-1})  ~p_{\theta}(x_{t-1:t}|y_{1:T})\textup{d}x_{t-1:t} \notag. 
\end{align}
\end{proposition}
\vspace{-0.5cm}
Hence, whereas we have $\nabla_{\theta}\ell^{\textup{ELBO}}(\theta,\phi)\rightarrow \nabla_{\theta}\ell(\theta)$ as $N \rightarrow \infty$ under regularity assumptions, the asymptotic bias of $\hat{\nabla}_{\theta}\ell^{\textup{ELBO}}(\theta,\phi)$ only vanishes if $p_{\theta}(x_{t-1:t}|y_{1:t})=p_{\theta}(x_{t-1:t}|y_{1:T})$; i.e. for models where the $X_t$ are independent. When $y_{t+1:T}$ do not bring significant information about $X_t$ given $y_{t:T}$, as for the models considered in \cite{le2017auto,maddison2017filtering,naesseth2017variational}, this is a reasonable approximation which explains the good performance reported therein. However, we show in \cref{sec:experiments} that this bias can also lead practically to inaccurate parameter estimation.

\subsection{Quantitative Bounds on the DET} 
 Weak convergence results for the ET have been established in \cite{reich2012nonparametric,myers2019sequential} and the DET in \cite{seguy2017large}. We provide here the first quantitative bound for the ET ($\epsilon=0$) and DET ($\epsilon>0$) which holds for any $N\geq 1$ by building upon results of \citep{li2020quantitative} and \citep{weed2018explicit}. We use the notation $\nu(\psi)\coloneqq\int \psi(x) \nu(\textup{d}x)$ for any measure $\nu$ and function $\psi$.
   
 \begin{proposition}\label{prop:CVDETDetailed}
           Consider atomic probability measures $\alpha_N=\sum_{i=1}^N a_i \delta_{Y^i}$ with $a_i>0$ and $\beta_N=\sum_{i=1}^N b_i \delta_{X^i}$, with support $\mathcal{X}\subset \mathbb{R}^d$. Let $\tilde{\beta}_N=\sum_{i=1}^N a_i \delta_{\tilde{X}^{i}_{N,\epsilon}}$ where $\tilde{\mathbf{X}}_{N,\epsilon}= \Delta^{-1} \mathbf{P}^{\mathrm{OT}}_\epsilon \mathbf{X}$ for $\Delta=\textup{diag}(a_1,...,a_N)$ and $\mathbf{P}^{\mathrm{OT}}_\epsilon$ is the transport matrix corresponding to the $\epsilon$-regularized OT coupling, $\mathcal{P}^{\mathrm{OT},N}_\epsilon$, between $\alpha_N$ and $\beta_N$. Let $\alpha, \beta$ be two other probability measures, also supported on $\mathcal{X}$, such that there exists a unique $\lambda$-Lipschitz optimal transport map $\mathbf{T}$ between them. Then for any bounded $1$-Lipschitz function $\psi$, we have
           \begin{multline}\label{eq:UpperBoundonDET}
     \left|\beta_N(\psi)-\tilde{\beta}_N(\psi)\right|
                    \leq 2\lambda^{1/2} 
        \mathcal{E}^{1/2}\left[\mathfrak{d}^{1/2}+ \mathcal{E}\right]^{1/2}\\
        + \max\{\lambda,  1\}\left[\wass_2(\alpha_N, \alpha)+  \wass_2(\beta_N, \beta)\right],
            \end{multline}
    where $\mathfrak{d}\coloneqq \sup_{x,y\in \mathcal{X}} |x-y|$ and $\mathcal{E}= \wass_2(\alpha_N, \alpha)+  \wass_2(\beta_N, \beta) + \sqrt{2\epsilon \log N}.$
\end{proposition}  
If $\wass_2(\alpha_N, \alpha), \wass_2(\beta_N, \beta)\to 0$ and we choose $\epsilon_N=o(1/\log N)$ the bound given in (\ref{eq:UpperBoundonDET}) vanishes with $N\to \infty$.
This suggested dependence of $\epsilon$ on $N$ comes from the entropic radius, see Lemma~\ref{lem:entropic_radius} in the Supplementary and \cite{weed2018explicit}, and is closely related to the fact that entropy-regularized OT is sensitive to the scale of $\mathbf{X}$. Equivalently one may rescale  $\mathbf{X}$ by a factor $\log N$ when computing the cost matrix.
In particular when $\alpha_N$ and $\beta_N$ are Monte Carlo approximations of $\alpha$ and $\beta$, we expect $\wass_2(\alpha_N, \alpha), \wass_2(\beta_N, \beta) =O(N^{-1/d})$ with high probability \cite{fournier2015rate}.

\subsection{Consistency of DPF}\label{subsec:consistencyDPF} The parameters $\theta,\phi$ are here fixed and omitted from notation. We now establish consistency results for DPF, showing that both the resulting particle approximations $\tilde{\beta}^{(t)}_N=\frac{1}{N}\sum_{i=1}^N\delta_{\tilde{X}^{i}_t}$ of $\beta^{(t)}=p(x_t|y_{1:t})$ and the corresponding log-likelihood approximation $\log \hat{p}_N(y_{1:T})$ of  $\log p(y_{1:T})$ are consistent. In the interest of simplicity, we limit ourselves to the scenario where the proposal is the transition, $q=f$, so $\omega(x_{t-1},x_t,y_t)=g(y_t|x_t)$, known as the bootstrap PF and study a slightly non-standard version of it proposed in \cite{del2001stability}; see
\cref{sec:ProofconsistencyDPF} for details. Consistency is established under regularity assumptions detailed in the Supplementary. Assumption~\ref{ass:compact} is that the space $\mathcal{X} \subset \mathbb{R}^d$ has a finite diameter $\mathfrak{d}$. Assumption~\ref{ass:lipcontraction} implies that the proposal mixes exponentially fast in the Wasserstein sense at a rate $\kappa$, which is reasonable given compactness, and essential for the error to not accumulate. Assumption~\ref{ass:omega} assumes a bounded importance weight function i.e. $g(y_t|x_t)\in [\Delta,\Delta^{-1}]$, again not unreasonable given compactness. Assumption~\ref{ass:lipschitz} states that at each time step, the optimal transport problem between $\alpha^{(t)}$ and $\beta^{(t)}$ is solved uniquely by a deterministic, globally Lipschitz map. Uniqueness is crucial for the quantitative stability results provided in the following proposition.


\begin{proposition}\label{prop:bias}
Under Assumptions \ref{ass:compact}, \ref{ass:lipcontraction}, \ref{ass:omega} and \ref{ass:lipschitz}, for any $\delta>0$, with probability at least $1-2\delta$ over the sampling steps, for any bounded 1-Lipschitz $\psi$, for any $t\in [1:T]$,  the approximations of the filtering distributions and log-likelihood computed by the bootstrap DPF satisfy
\begin{align*}
                |\tilde{\beta}^{(t)}_{N}(\psi)-\beta^{(t)}(\psi)
               | &\leq  \mathfrak{G}_{\epsilon, \delta/T, N, d}^{(t)}\left(\lambda(c,C, d,T,N,\delta)\right),
\end{align*}
\begin{align*}
        \left|
         \log \frac{\hat{p}_N(y_{1:T})}{p(y_{1:T})}\right|  &\leq \frac{\kappa}{\Delta} \max_{t\in [1:T]}\mathrm{Lip}\left[g(y_t\mid \cdot) \right]\\
    & \times \sum_{t=1}^T  \mathfrak{G}_{\epsilon, \delta/T, N, d}^{(t)}\left(\lambda(c,C,d,T,N,\delta)\right),
\end{align*}
for $\lambda(c,C,d,T,N,\delta)=\sqrt{f_d^{-1} \left(\frac{\log(CT/\delta)}{cN} \right)}$ where $c,C$ are finite constants independent of $T,$ and $\mathrm{Lip}[f]$ is the Lipschitz constant of the function $f,$ and $\mathfrak{G}^{(t)}_{N, \epsilon}, f_d$ defined in Appendix \ref{sec:ProofconsistencyDPF} are two functions such that if we set $\epsilon_N=o(1/\log N)$ then we have in probability
    \begin{equation*}
      |\tilde{\beta}^{(t)}_{N}(\psi)-\beta^{(t)}(\psi)|  \to 0,\quad
      \left| \log \frac{\hat{p}_N(y_{1:T})}{p(y_{1:T})}\right| \to 0.
    \end{equation*}
\end{proposition}

The above bounds are certainly not sharp. A glimpse into the behavior of the above bounds in terms of $T$ can be obtained through careful consideration of the quantities appearing in Proposition~\ref{prop:recursivebound} in the supplement.
In particular, for $\kappa$ small enough, it suggests that the bound on the error of the log-likelihood estimator grows linearly with $T$ as for standard PF under mixing assumptions. Sharper bounds are certainly possible, e.g. using a $L_1$ version of Theorem~3.5 in \cite{li2020quantitative}. It would also be of interest to weaken the assumptions, in particular, to remove the bounded space assumption although it is very commonly made in the PF literature to obtain quantitative bounds; see e.g. \cite{delmoral2004,douc2014nonlinear}. Although this is not made explicit in the expressions above, there is an exponential dependence of the bounds on the state dimension $d_x$. This is unavoidable however and a well-known limitation of PF methods.

Finally note that DPF provides a biased estimate of the likelihood contrary to standard PF, so we cannot guarantee that the expectation of its logarithm, $\ell_{\epsilon}^{\mathrm{ELBO}}(\theta,\phi)\coloneqq\mathbb{E}_{\mathbf{U}}[\hat{\ell}_{\epsilon}(\theta;\phi,\mathbf{U})]$. is actually a valid ELBO. However in all our experiments, see e.g. \cref{subsec:LGM}, $|\ell^{\mathrm{ELBO}}_{\epsilon}(\theta,\phi) -\ell^{\mathrm{ELBO}}(\theta,\phi)|$ is significantly smaller than $\ell(\theta)-\ell^{\mathrm{ELBO}}(\theta,\phi)$ so $\ell^{\mathrm{ELBO}}_{\epsilon}(\theta,\phi)<\ell(\theta)$. Hence we keep the ELBO terminology.

\section{Experiments}\label{sec:experiments}
In \cref{subsec:LGM}, we assess the sensitivity of the DPF to the regularization parameter $\epsilon$. All other DPF experiments presented here use the DET Resampling detailed in \cref{algo:differentiableResampling} with $\epsilon=0.5$, which ensures stability of the gradient calculations while adding little bias to the calculation of the ELBO compared to standard PF. Our method is implemented in both PyTorch and TensorFlow, the code to replicate the experiments as well as further experiments may be found at \url{https://github.com/JTT94/filterflow}.

 \subsection{Linear Gaussian State-Space Model}\label{subsec:LGM}
We consider here a simple two-dimensional linear Gaussian SSM for which the exact likelihood can be computed exactly using the Kalman filter
\begin{align*}
    X_{t+1}|\{ X_{t}=x \}&\sim \mathcal{N} \left(\text{diag}(\theta_1~\theta_2) x, 0.5 \mathbf{I}_2\right),\\ Y_t|\{ X_t=x\}&\sim \mathcal{N}(x, 0.1 \mathbf{I}_2).
\end{align*}
We simulate $T=150$ observations using $\theta=(\theta_1,\theta_2)=(0.5,0.5)$, for which we evaluate the ELBO at $\theta=(0.25,0.25)$, $\theta=(0.5,0.5)$, and $\theta=(0.75,0.75)$. More precisely, using a standard PF with $N=25$ particles, we compute the mean and standard deviation of $\frac{1}{T}(\hat{\ell}(\theta;\mathbf{U})-\ell(\theta))$ over 100 realizations of $\bf{U}$. The mean is an estimate of the ELBO minus the true log-likelihood (rescaled by $1/T$). We then perform the same calculations for the DPF using the same number of particles and $\epsilon=0.25,0.5,0.75$. As mentioned in \cref{sec:differentiable_et} and \cref{subsec:consistencyDPF}, the DET resampling scheme is only satisfying \cref{resamplingunbiased} for affine functions $\psi$ so the DPF provides a biased estimate of the likelihood. Hence we cannot guarantee that the expectation of the corresponding log-likelihood estimate is a true ELBO. However, from \cref{tab:estimatesVariance}, we observe that the difference between the ELBO estimates computed using PF and DPF is negligible for the three values of $\epsilon$. The standard deviation of the log-likelihood estimates is also similar.
\vspace{-0.5cm}
\begin{table}[H]
    \centering
    \caption{Mean \& std of $\frac{1}{T} (\hat{\ell}(\theta;\mathbf{U})-\ell(\theta))$}
    

 \begin{tabular}{rrrrr}
 \toprule
   \multicolumn{2}{c}{$\theta_1$, $\theta_2$}           & 0.25  & 0.5   & 0.75 \\
   \midrule
   \multirow{2}{*}{PF}                           & mean & -1.13 & -0.93 & -1.05 \\
                                                 & std  & 0.20  & 0.18  & 0.17 \\
   \midrule
   \multirow{2}{*}{DPF ($\epsilon=0.25$)}        & mean & -1.14 & -0.94 & -1.07 \\
                                                 & std  & 0.20  & 0.18  & 0.19 \\
   \midrule
   \multirow{2}{*}{DPF ($\epsilon=0.5$)}         & mean & -1.14 & -0.94 & -1.08 \\
                                                 & std  & 0.20  & 0.18  & 0.18 \\
   \midrule
   \multirow{2}{*}{DPF ($\epsilon=0.75$)}        & mean & -1.14 & -0.94 & -1.08 \\
                                                 & std  & 0.20  & 0.18  & 0.18 \\
 \bottomrule
 \end{tabular}



\label{tab:estimatesVariance}
\end{table}
Recall here that alternative techniques estimating the score vector $\nabla_\theta \ell(\theta)$ by approximating (\ref{eq:scoreidentity}) using particle smoothing algorithms \cite{poyiadjis2011particle,kantas2015particle} could also be used to estimate $\theta$.

\subsection{Learning the Proposal Distribution}
    We consider a similar example as in \cite{naesseth2017variational} where one learns the parameters $\phi$ of the proposal using the ELBO for the following linear Gaussian SSM:
    \begin{align}\label{eq:proposalLGSS1}
        X_{t+1}|\{X_{t}=x\}\sim \mathcal{N}\left(\mathbf{A}x, \mathbf{I}_{d_x}\right), \\ Y_t|\{ X_t=x \}\sim \mathcal{N}( \mathbf{I}_{d_y, d_x} x, \mathbf{I}_{d_y}),\label{eq:proposalLGSS2}
    \end{align} with $\mathbf{A} = (0.42^{|i-j|+1})_{1 \leq i,j\leq d_x}$, $\mathbf{I}_{d_y, d_x}$ is a $d_y\times d_x$ matrix with $1$ on the diagonal for the $d_y$ first rows and zeros elsewhere. For $\mathbf{\phi} \in \mathbb{R}^{d_x+d_y}$, we consider
    \begin{equation*}
        \label{eq:learned_proposal}
        q_\phi(x_t|x_{t-1}, y_t)=\mathcal{N}(x_t|\Delta_{\phi}^{-1}\left(\mathbf{A} x_{t-1} + \Gamma_{\phi} y_t\right), \Delta_\phi),
    \end{equation*}
    with $\Delta_{\phi} = \text{diag}(\phi_1, \dots, \phi_{d_x})$ and a $d_x \times d_y$ matrix $\Gamma_{\phi} = \text{diag}_{d_x, d_y}(\phi_1, \dots, \phi_{d_x})$ with $\phi_i$ on the diagonal for $d_x$ first rows and zeros elsewhere.
    The locally optimal proposal $p(x_t|x_{t-1},y_t)\propto g(y_t|x_t)f(x_t|x_{t-1})$ in \cite{doucet2009tutorial} corresponds to $\mathbf{\phi} = \mathbf{1}$, the vector with unit entries of dimension $d_\phi=d_x+d_y$.
    
    For $d_x=25,d_y=1$, $M=100$ realizations of $T=100$ observations using (\ref{eq:proposalLGSS1})-(\ref{eq:proposalLGSS2}), we learn $\phi$ on each realization using 100 steps of stochastic gradient ascent with learning rate $0.1$ on the $\ell^{\text{ELBO}}(\phi)$ using regular PF with biased gradients as in \cite{maddison2017filtering,le2017auto,naesseth2017variational} and $\ell^{\text{ELBO}}(\phi)$ with four independent filters using DPF. We use $N=500$ for regular PF and $N=25$ for DPF so as to match the computational complexity. While $p(x_t|x_{t-1},y_t)$ is not guaranteed to maximize the ELBO, our experiments showed that it outperforms optimized proposals. We therefore report the RMSE of $\mathbf{\phi} - \mathbf{1}$ and the average Effective Sample Size (ESS) \cite{doucet2009tutorial} as proxy performance metrics. On both metrics, DPF outperforms regular PF. The RMSE over 100 experiments is 0.11 for DPF vs 0.22 for regular PF while the average ESS after convergence is around 60\% for DPF vs 25\% for regular PF. The average time per iteration was around 15 seconds for both DPF and PF.

    \subsection{Variational Recurrent Neural Network (VRNN)}
    A VRNN is an SSM introduced by \cite{chung2015recurrent} to improve upon LSTMs (Long Short Term Memory networks) with the addition of a stochastic component to the hidden state, this extends variational auto-encoders to a sequential setting. Indeed let latent state be $X_t=(R_t,Z_t)$ where $R_t$ is an RNN state and $Z_t$ a latent Gaussian variable, here $Y_t$ is a vector of binary observations. The VRNN is detailed as follows. $\text{RNN}_{\theta}$ denotes the forward call of an LSTM cell which at time $t$ emits the next RNN state $R_{t+1}$ and output $O_{t+1}$. $E_\theta$, $h_\theta$, $\mu_\theta$, $\sigma_\theta$ are fully connected neural networks; detailed fully in the Supplementary Material. This model is trained on the polyphonic music benchmark datasets \cite{boulangerlew2012modeling}, whereby $Y_t$ represents which notes are active. The observation sequences are capped to length $150$ for each dataset, with each observation of dimension $88$. We chose latent states $Z_t$ and $R_t$ to be of dimension $d_z=8$ and $d_r=16$ respectively so $d_x=24$. We use $q_{\phi}(x_t|x_{t-1},y_t)=f_{\theta}(x_t|x_{t-1})$.
        \begin{align*}
            (R_{t+1}, O_{t+1}) &= \text{RNN}_\theta(R_{t}, Y_{1:t}, E_\theta(Z_{t})),\\
            Z_{t+1} &\sim \mathcal{N}(\mu_{\theta}(O_{t+1}),
             \sigma_{\theta}(O_{t+1})), \\
            \hat{p}_{t+1} &= h_\theta(E_\theta(Z_{t+1}), O_{t+1}), \\ 
             Y_t | X_t  &\sim \text{Ber}(\hat{p}_t).
        \end{align*}
        
    \begin{table}[h]
    \vspace{-0.7cm}
        \caption{ELBO $\pm$ Standard Deviation evaluated using Test Data.}
        \label{tab:music}
        \begin{center}
        \begin{small}
        \begin{sc}
        \begin{tabular}{lcccr}
        \toprule
            & MuseData & JSB   & Nottingham  \\
            \midrule
    {DPF} & $-\mathbf{7.59}_{\pm 0.01 }$   & $-\mathbf{7.67}_{\pm 0.08}$ & $\mathbf{-3.79}_{\pm 0.02}$      \\
    {PF}  & $-7.60_{\pm 0.06}$            & $-7.92_{\pm 0.13}$          & $-3.81_{\pm 0.02}$     \\
    {SPF} & $-7.73_{\pm 0.14}$            & $-8.17_{\pm 0.07}$          & $-3.91_{\pm 0.05}$        \\
            \bottomrule
            \end{tabular}
            \end{sc}
            \end{small}
            \end{center}
            \end{table}
        
    The VRNN model is trained by maximizing $\ell_{\epsilon}^{\text{ELBO}}(\theta)$ using DPF. We compare this to the same model trained by maximizing  $\ell^{\text{ELBO}}(\theta)$ computed with regular PF \cite{maddison2017filtering} and also trained with `soft-resampling' (SPF) introduced by \cite{karkus2018particle} and described in Section \ref{sec:related},
    SPF is used here with parameter $\alpha=0.1$. Unlike regular resampling, SPF partially preserves a gradient through the resampling step, however SPF still involves a non-differentiable operation, again resulting in a biased gradient. SPF also produces higher variance estimates as the resampled approximation is not uniformly weighted, essentially interpolating between PF and IWAE. Each of the methods are performed with $N=32$ particles. Although DET is computationally more expensive than the other resampling schemes, the computational times of DPF, PF, and SPF are very similar due to most of the complexity coming from neural network operations. The learned models are then evaluated on test data using multinomial resampling for comparable ELBO results. Due to the fact that our observation model is $\text{Ber}(\hat{p}_t)$, this recovers the negative log-predictive cross-entropy.
    
     Table \ref{tab:music} illustrates the benefit of using DPF over regular PF and SPF for the JSB dataset. Although DPF remains competitive compared to other heuristic approaches, the difference is relatively minor for the other datasets. We speculate that the performance of the heuristic methods is likely due to low predictive uncertainty for the next observation given the previous one.
    
    
    
    \subsection{Robot Localization}
     Consider the setting of a robot/agent in a  maze \cite{jonschkowski2018differentiable,karkus2018particle}. Given the agent's initial state, $S_1$, and inputs $a_t$, one would like to infer the location of the agent at any specific time given observations $O_t$. Let the latent state be denoted $S_t=(X^{(1)}_t, X^{(2)}_t, \gamma_t)$ where $(X^{(1)}_t,X^{(2)}_t)$ are location coordinates and $\gamma_t$ the robot's orientation. In our setting observations $O_t$ are images, which are encoded to extract useful features using a neural network $E_\theta$, where $Y_t = E_\theta(O_t)$. This problem requires learning the relationship between the robot's location, orientation and the observations. Given actions $a_t=(v^{(1)}_t, v^{(2)}_t, \omega_t)$, we have 
       \begin{align*}
           S_{t+1} &= F_\theta(S_t, a_t) + \nu_t,\quad  \nu_t \overset{\textup{i.i.d.}}{\sim} \mathcal{N}(\mathbf{0}, \Sigma_F),\\
           Y_{t} &= G_\theta(S_t) + \epsilon_t,\quad \epsilon_t \overset{\textup{i.i.d.}}{\sim}\mathcal{N}(\mathbf{0}, \sigma_{G}^2\mathbb{I}_{e_d}),
       \end{align*}
where $\Sigma_F=\textup{diag}(\sigma_x^2,\sigma_x^2,\sigma_\theta^2 )$ and the relationship between state $S_t$ and image encoding $Y_t$ may be parameterized by another neural network $G_\theta$. We consider here a simple linear model of the dynamics
\[
    F(S_t,a_t)= 
\begin{bmatrix}
    X^{(1)}_t + v^{(1)}_t \cos(\gamma_t) + v^{(2)}_t \sin(\gamma_t)\\
    X^{(2)}_t + v^{(1)}_t \sin(\gamma_t) - v^{(2)}_t \cos(\gamma_t)\\
    \gamma_t + \omega_t
\end{bmatrix}.
\]
$D_\theta$ denotes a decoder neural network, mapping the encoding back to the original image. $E_\theta$, $G_\theta$ and $D_\theta$ are trained using a loss function consisting of the PF-estimated log-likelihood $\hat{\mathcal{L}}_\mathrm{PF}$; PF-based mean squared error (MSE), $\hat{\mathcal{L}}_\mathrm{MSE}$; and auto-encoder loss, $\hat{\mathcal{L}}_\mathrm{AE}$, given per-batch as in \cite{wen2020end}:
    \begin{align*}
        \hat{\mathcal{L}}_\mathrm{MSE} &\coloneqq \frac{1}{T}\sum_{t=1}^T||X^*_t-\sum_{i=1}^Nw^i_tX^i_t||^2,~~ \hat{\mathcal{L}}_\mathrm{PF}:= -\frac{1}{T}\hat{\ell}(\theta),\\
        \hat{\mathcal{L}}_\mathrm{AE} &\coloneqq \sum_{t=1}^T||D_\theta(E_\theta(O_t))- O_t||^2,
    \end{align*}
where $X^{\star}_t$ are the true states available from training data and $\sum_{i=1}^Nw^i_tX^i_t$ are the PF estimates of $\mathbb{E}[X_t|y_{1:t}]$. The auto-encoder / reconstruction loss $\hat{\mathcal{L}}_\mathrm{AE}$ ensures the encoder is informative and prevents the case whereby networks $G_\theta$, $E_\theta$ map to a constant. The PF-based loss terms $\hat{\mathcal{L}}_\mathrm{MSE}$ and $\hat{\mathcal{L}}_\mathrm{PF}$ are not differentiable w.r.t. $\theta$ under traditional resampling schemes.
   \begin{figure}
            \centering
            \includegraphics[width=\linewidth]{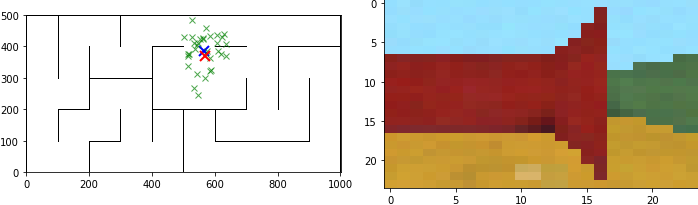}
            \caption{Left: Particles $(X^{(1),i}_t, X^{(2),i}_t)$ (green), PF estimate of $\mathbb{E}[X_t|y_{1:t}]$ (blue), true state $X_t^*$ (red). Right: Observation, $O_t$.}
            \label{fig:map_image}
    \end{figure}
    \begin{table}
        \caption{MSE and $\pm$ Standard Deviation evaluated on Test Data: Lower is better}
        \label{tab:maze_res}
        \vskip 0.15in
        \begin{center}
        \begin{small}
        \begin{sc}
        \begin{tabular}{lcccr}
        \toprule
                 & Maze 1                             & Maze 2                             & Maze 3                             \\
                 \midrule
            DPF  & $\mathbf{3.55}_{\pm \mathbf{0.20}}$  & $\mathbf{4.65}_{\pm 0.50}$  & $\mathbf{4.44}_{\pm \mathbf{0.26}}$  \\
            PF  & $10.71_{\pm 0.45}$ & $11.86_{\pm 0.57}$ & $12.88_{\pm 0.65}$ \\
            SPF & $9.14_{\pm 0.39}$ & $10.12_{\pm \mathbf{0.40}}$                    & $11.42_{\pm 0.37}$   \\ 
            \bottomrule
            \end{tabular}
            \end{sc}
            \end{small}
            \end{center}
            \vskip -0.1in
        \end{table}
        
We use the setup from \cite{jonschkowski2018differentiable} with data from DeepMind Lab \cite{beattie2016deepmind}. This consists of $3$ maze layouts of varying sizes. We have access to `true' trajectories of length $1,000$ steps for each maze. Each step has an associated state, action and observation image, as described above. The visual observation $O_t$ consists of $32\times32$ RGB pixel images, compressed to $24 \times 24$, as shown in Figure \ref{fig:map_image}. Random, noisy subsets of fixed length are sampled at each training iteration. To illustrate the benefits of our proposed method, we select the random subsets to be of length $50$ as opposed to length $20$ as chosen in \cite{jonschkowski2018differentiable}. Training details in terms of learning rates, number of training steps and neural network architectures for $E_\theta$, $G_\theta$ and $D_\theta$ are given in the Appendices.
 
We compare our method, DPF, to regular PF used in \cite{maddison2017filtering} and Soft PF (SPF) used in \cite{karkus2018particle, ma2019particle, ma2020discriminative}, whereby the soft resampling is used with $\alpha=0.1$. As most of the computational complexity arises from neural network operations, DPF is of similar overall computational cost to SPF and PF. As shown in Table \ref{tab:maze_res} and Figure \ref{fig:MSE}, DPF significantly outperforms previously considered PF methods in this experiment. The observation model becomes increasingly important for longer sequences due to resampling and weighting operations. Indeed, as shown in Figure \ref{fig:maps}, the error is small for both models at the start of the sequence, however the error at later stages in the sequence is visibly smaller for the model trained using DPF. 

        \begin{figure}[h]
        \centering
        \begin{subfigure}{0.32\linewidth}
            \centering
            \includegraphics[width=\linewidth]{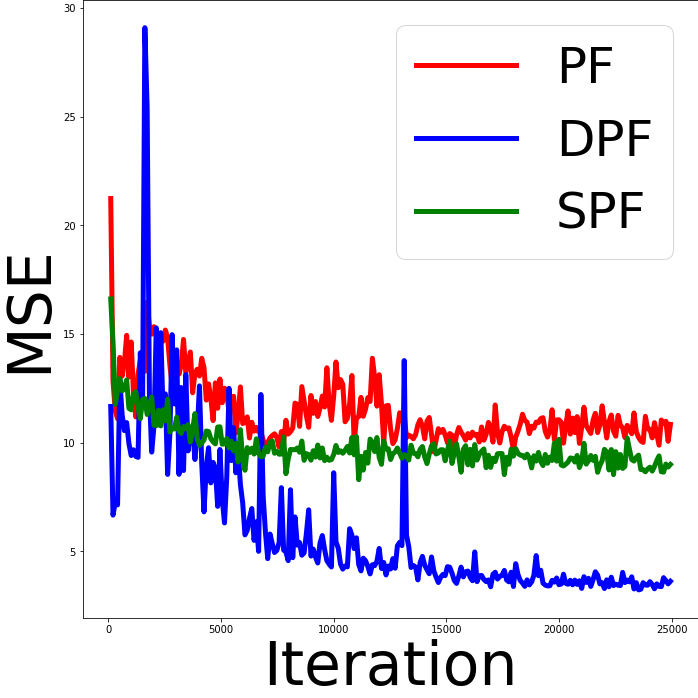}
            \caption{Maze 1}
        \end{subfigure}
        \begin{subfigure}{0.32\linewidth}
            \centering
            \includegraphics[width=\linewidth]{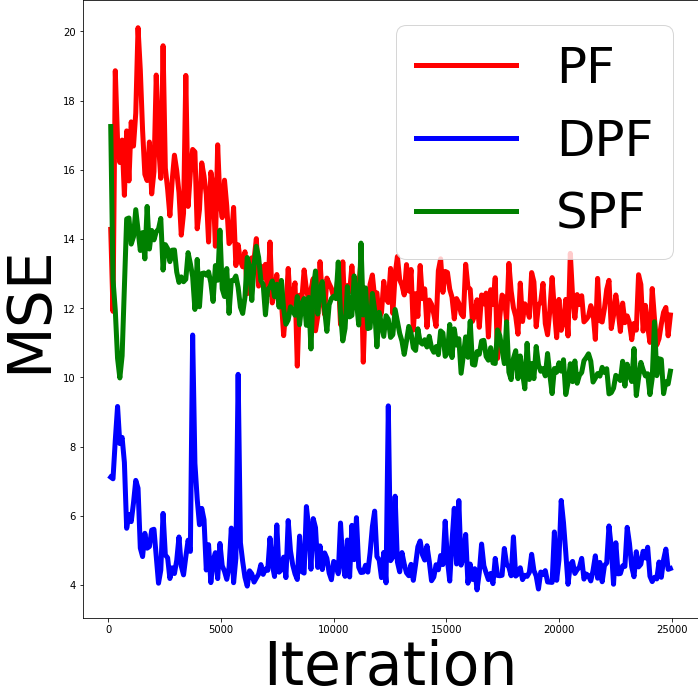}
            \caption{Maze 2}
        \end{subfigure}
        \begin{subfigure}{0.32\linewidth}
            \centering
            \includegraphics[width=\linewidth]{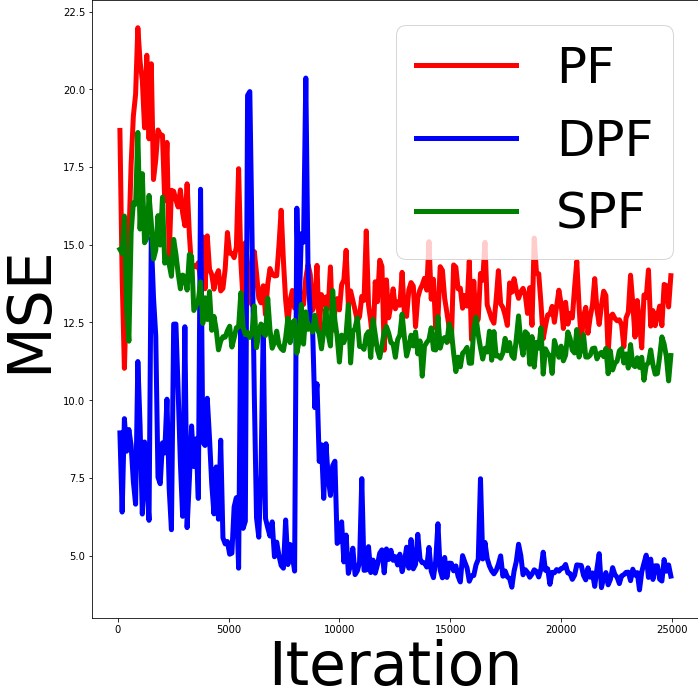}
            \caption{Maze 3}
        \end{subfigure}
        \caption{MSE of PF (red), SPF (green) and DPF (blue) estimates, evaluated on test data during training.}
        \label{fig:MSE}
    \end{figure}
    \begin{figure}[h]
    \centering
        \begin{subfigure}{0.49\linewidth}
            \centering
            \includegraphics[width=\linewidth]{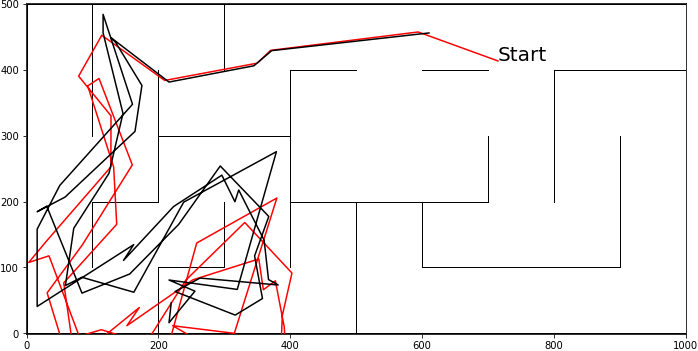}
            \caption{Standard PF}
        \end{subfigure}
        \begin{subfigure}{0.49\linewidth}
            \centering
            \includegraphics[width=\linewidth]{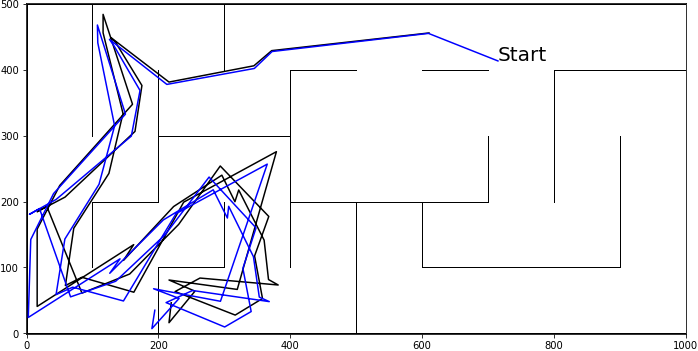}
            \caption{Differentiable PF}
        \end{subfigure}
        \caption{Illustrative Example: PF estimate of path compared to true path (black) on a single $50$-step trajectory from test data.}
        \label{fig:maps}
    \end{figure}

\section{Discussion}

    

This paper introduces the first principled, fully differentiable PF (DPF) which can use general proposal distributions. It provides a differentiable estimate of the log-likelihood function and more generally differentiable estimates of PF-based losses. This permits parameter inference in state-space models and proposal distributions, using end-to-end gradient based optimization. This also allows the use of PF routines in general differentiable programming pipelines, in particular as a differentiable sampling method for inference in probabilistic programming languages \cite{dillon2017tensorflow,hong2018turing, vandemeent2018introduction}.

For a given number of particles $N$, existing PF methods ignoring resampling gradient terms have computational complexity $O(N)$. Training with these resampling schemes however is unreliable and performance cannot be improved by increasing $N$ as gradient estimates are inconsistent and the limiting bias can be significant. DPF has complexity $O(N^2)$ during training. However, this cost is dwarfed when training large neural networks. Additionally, once the model is trained, standard PF may be ran at complexity $O(N)$. The benefits of DPF are confirmed by our experimental results where it was shown to outperform existing techniques, even when an equivalent computational budget was used. Moreover, recent techniques have been proposed to speed up the Sinkhorn algorithm \cite{altschuler2019massively,scetbon2020linear} at the core of DPF and could potentially be used here to reduce its complexity. 

Regularization parameter $\epsilon$ was not fine-tuned in our experiments. In future work, it would be interesting to obtain sharper quantitative bounds on DPF to propose principled guidelines on choosing $\epsilon,$ further improving its performance. Finally, we have focused on the use of the differentiable ensemble transform to obtain a differentiable resampling scheme. However, alternative OT approaches could also be proposed such as a differentiable version of the second order ET presented in \cite{acevedo2017second}, techniques based on point cloud optimization \cite{cuturidoucet2014,peyr2019computational} relying on the Sinkhorn divergence \cite{genevay2018learning} or the sliced-Wasserstein metric. Alternative non-entropic regularizations, such as the recently proposed Gaussian smoothed OT \cite{goldfeld2020gaussiansmooth}, could also lead to DPFs of interest.

\section*{Acknowledgments}
Adrien Corenflos was supported by the Academy of Finland (projects 321900 and 321891). Arnaud Doucet is supported by the EPSRC CoSInES (COmputational Statistical INference for Engineering and Security) grant EP/R034710/1, James Thornton by the OxWaSP CDT through grant EP/L016710/1. Computing resources were provided through the Google Cloud Research Credits Programme. The authors thank Valentin De Bortoli and Yuyang Shi for their comments and Laurence Aitchison for bringing attention to reference \cite{aitchison2018tensor}.

\bibliography{main.bib}

\begin{thebibliography}{61}
\providecommand{\natexlab}[1]{#1}
\providecommand{\url}[1]{\texttt{#1}}
\expandafter\ifx\csname urlstyle\endcsname\relax
  \providecommand{\doi}[1]{doi: #1}\else
  \providecommand{\doi}{doi: \begingroup \urlstyle{rm}\Url}\fi

\bibitem[Acevedo et~al.(2017)Acevedo, de~Wiljes, and Reich]{acevedo2017second}
Acevedo, W., de~Wiljes, J., and Reich, S.
\newblock Second-order accurate ensemble transform particle filters.
\newblock \emph{SIAM Journal on Scientific Computing}, 39\penalty0
  (5):\penalty0 A1834--A1850, 2017.

\bibitem[Aitchison(2019)]{aitchison2018tensor}
Aitchison, L.
\newblock Tensor {M}onte {C}arlo: particle methods for the {GPU} era.
\newblock \emph{Advances in Neural Information Processing Systems}, 2019.

\bibitem[Altschuler et~al.(2017)Altschuler, Niles-Weed, and
  Rigollet]{altschuler2017near}
Altschuler, J., Niles-Weed, J., and Rigollet, P.
\newblock Near-linear time approximation algorithms for optimal transport via
  {S}inkhorn iteration.
\newblock In \emph{Advances in Neural Information Processing Systems}, pp.\
  1964--1974, 2017.

\bibitem[Altschuler et~al.(2019)Altschuler, Bach, Rudi, and
  Niles-Weed]{altschuler2019massively}
Altschuler, J., Bach, F., Rudi, A., and Niles-Weed, J.
\newblock Massively scalable {S}inkhorn distances via the {N}ystr{\"o}m method.
\newblock In \emph{Advances in Neural Information Processing Systems}, pp.\
  4429--4439, 2019.

\bibitem[Archer et~al.(2015)Archer, Park, Buesing, Cunningham, and
  Paninski]{archer2015black}
Archer, E., Park, I.~M., Buesing, L., Cunningham, J., and Paninski, L.
\newblock Black box variational inference for state space models.
\newblock \emph{arXiv preprint arXiv:1511.07367}, 2015.

\bibitem[Beattie et~al.(2016)Beattie, Leibo, Teplyashin, Ward, Wainwright,
  Küttler, Lefrancq, Green, Valdés, Sadik, Schrittwieser, Anderson, York,
  Cant, Cain, Bolton, Gaffney, King, Hassabis, Legg, and
  Petersen]{beattie2016deepmind}
Beattie, C., Leibo, J.~Z., Teplyashin, D., Ward, T., Wainwright, M., Küttler,
  H., Lefrancq, A., Green, S., Valdés, V., Sadik, A., Schrittwieser, J.,
  Anderson, K., York, S., Cant, M., Cain, A., Bolton, A., Gaffney, S., King,
  H., Hassabis, D., Legg, S., and Petersen, S.
\newblock Deep{M}ind {L}ab, 2016.

\bibitem[Bertsimas \& Tsitsiklis(1997)Bertsimas and
  Tsitsiklis]{bertsimas1997introduction}
Bertsimas, D. and Tsitsiklis, J.~N.
\newblock \emph{Introduction to Linear Optimization}.
\newblock Athena Scientific Belmont, MA, 1997.

\bibitem[Boulanger-Lewandowski et~al.(2012)Boulanger-Lewandowski, Bengio, and
  Vincent]{boulangerlew2012modeling}
Boulanger-Lewandowski, N., Bengio, Y., and Vincent, P.
\newblock Modeling temporal dependencies in high-dimensional sequences:
  Application to polyphonic music generation and transcription.
\newblock In \emph{International Conference on Machine Learning}, pp.\
  1881--1888, 2012.

\bibitem[Burda et~al.(2016)Burda, Grosse, and
  Salakhutdinov]{burda2016importance}
Burda, Y., Grosse, R.~B., and Salakhutdinov, R.
\newblock Importance weighted autoencoders.
\newblock In \emph{International Conference on Learning Representations}, 2016.

\bibitem[Chopin \& Papaspiliopoulos(2020)Chopin and
  Papaspiliopoulos]{chopin2020introduction}
Chopin, N. and Papaspiliopoulos, O.
\newblock \emph{An Introduction to {S}equential {M}onte {C}arlo}.
\newblock Springer, 2020.

\bibitem[Chung et~al.(2015)Chung, Kastner, Dinh, Goel, Courville, and
  Bengio]{chung2015recurrent}
Chung, J., Kastner, K., Dinh, L., Goel, K., Courville, A.~C., and Bengio, Y.
\newblock A recurrent latent variable model for sequential data.
\newblock In \emph{Advances in Neural Information Processing Systems}, pp.\
  2980--2988, 2015.

\bibitem[Cuturi(2013)]{cuturi2013sinkhorn}
Cuturi, M.
\newblock Sinkhorn distances: Lightspeed computation of optimal transport.
\newblock In \emph{Advances in Neural Information Processing Systems}, pp.\
  2292--2300, 2013.

\bibitem[Cuturi \& Doucet(2014)Cuturi and Doucet]{cuturidoucet2014}
Cuturi, M. and Doucet, A.
\newblock Fast computation of {W}asserstein barycenters.
\newblock In \emph{International Conference on Machine Learning}, pp.\
  685--693, 2014.

\bibitem[DeJong et~al.(2013)DeJong, Liesenfeld, Moura, Richard, and
  Dharmarajan]{dejong2013efficient}
DeJong, D.~N., Liesenfeld, R., Moura, G.~V., Richard, J.-F., and Dharmarajan,
  H.
\newblock Efficient likelihood evaluation of state-space representations.
\newblock \emph{Review of Economic Studies}, 80\penalty0 (2):\penalty0
  538--567, 2013.

\bibitem[Del~Moral(2004)]{delmoral2004}
Del~Moral, P.
\newblock \emph{Feynman-Kac Formulae}.
\newblock Springer, 2004.

\bibitem[Del~Moral \& Guionnet(2001)Del~Moral and Guionnet]{del2001stability}
Del~Moral, P. and Guionnet, A.
\newblock On the stability of interacting processes with applications to
  filtering and genetic algorithms.
\newblock In \emph{Annales de l'Institut Henri Poincar{\'e} (B) Probability and
  Statistics}, volume~37, pp.\  155--194, 2001.

\bibitem[Dillon et~al.(2017)Dillon, Langmore, Tran, Brevdo, Vasudevan, Moore,
  Patton, Alemi, Hoffman, and Saurous]{dillon2017tensorflow}
Dillon, J.~V., Langmore, I., Tran, D., Brevdo, E., Vasudevan, S., Moore, D.,
  Patton, B., Alemi, A., Hoffman, M., and Saurous, R.~A.
\newblock {T}ensorflow distributions.
\newblock \emph{arXiv preprint arXiv:1711.10604}, 2017.

\bibitem[Douc et~al.(2014)Douc, Moulines, and Stoffer]{douc2014nonlinear}
Douc, R., Moulines, E., and Stoffer, D.
\newblock \emph{Nonlinear Time Series: Theory, Methods and Applications with R
  Examples}.
\newblock CRC press, 2014.

\bibitem[Doucet \& Johansen(2009)Doucet and Johansen]{doucet2009tutorial}
Doucet, A. and Johansen, A.~M.
\newblock A tutorial on particle filtering and smoothing: Fifteen years later.
\newblock \emph{Handbook of Nonlinear Filtering}, 12:\penalty0 656--704, 2009.

\bibitem[Doucet \& Lee(2018)Doucet and Lee]{doucet2018sequential}
Doucet, A. and Lee, A.
\newblock Sequential {M}onte {C}arlo methods.
\newblock \emph{Handbook of Graphical Models}, pp.\  165--189, 2018.

\bibitem[Feydy et~al.(2019)Feydy, S{\'e}journ{\'e}, Vialard, Amari, Trouv{\'e},
  and Peyr{\'e}]{feydy2018interpolating}
Feydy, J., S{\'e}journ{\'e}, T., Vialard, F.-X., Amari, S.-I., Trouv{\'e}, A.,
  and Peyr{\'e}, G.
\newblock Interpolating between optimal transport and {MMD} using {S}inkhorn
  divergences.
\newblock In \emph{International Conference on Artificial Intelligence and
  Statistics}, 2019.

\bibitem[Finke et~al.(2016)Finke, Doucet, and Johansen]{finke2016embedded}
Finke, A., Doucet, A., and Johansen, A.~M.
\newblock On embedded hidden {M}arkov models and particle {M}arkov chain
  {M}onte {C}arlo methods.
\newblock \emph{arXiv preprint arXiv:1610.08962}, 2016.

\bibitem[Flamary et~al.(2018)Flamary, Cuturi, Courty, and
  Rakotomamonjy]{flamary2018wasserstein}
Flamary, R., Cuturi, M., Courty, N., and Rakotomamonjy, A.
\newblock {W}asserstein discriminant analysis.
\newblock \emph{Machine Learning}, 107\penalty0 (12):\penalty0 1923--1945,
  2018.

\bibitem[Fournier \& Guillin(2015)Fournier and Guillin]{fournier2015rate}
Fournier, N. and Guillin, A.
\newblock On the rate of convergence in {W}asserstein distance of the empirical
  measure.
\newblock \emph{Probability Theory and Related Fields}, 162\penalty0
  (3):\penalty0 707--738, 2015.

\bibitem[Ge et~al.(2018)Ge, Xu, and Ghahramani]{hong2018turing}
Ge, H., Xu, K.~X., and Ghahramani, Z.
\newblock {T}uring: A language for flexible probabilistic inference.
\newblock In \emph{International Conference on Artificial Intelligence and
  Statistics}, pp.\  1682--1690, 2018.

\bibitem[Genevay et~al.(2018)Genevay, Peyr{\'e}, and
  Cuturi]{genevay2018learning}
Genevay, A., Peyr{\'e}, G., and Cuturi, M.
\newblock Learning generative models with {S}inkhorn divergences.
\newblock In \emph{International Conference on Artificial Intelligence and
  Statistics}, pp.\  1608--1617, 2018.

\bibitem[Goldfeld \& Greenewald(2020)Goldfeld and
  Greenewald]{goldfeld2020gaussiansmooth}
Goldfeld, Z. and Greenewald, K.
\newblock Gaussian-smooth optimal transport: Metric structure and statistical
  efficiency.
\newblock \emph{arXiv preprint arXiv 2001.09206}, 2020.

\bibitem[Graves(2016)]{graves2016stochastic}
Graves, A.
\newblock Stochastic backpropagation through mixture density distributions.
\newblock \emph{arXiv preprint arXiv:1607.05690}, 2016.

\bibitem[Hirt \& Dellaportas(2019)Hirt and Dellaportas]{HirtDellaportas2019}
Hirt, M. and Dellaportas, P.
\newblock Scalable {B}ayesian learning for state space models using variational
  inference with {SMC} samplers.
\newblock In \emph{International Conference on Artificial Intelligence and
  Statistics}, pp.\  76--86, 2019.

\bibitem[Jonschkowski et~al.(2018)Jonschkowski, Rastogi, and
  Brock]{jonschkowski2018differentiable}
Jonschkowski, R., Rastogi, D., and Brock, O.
\newblock Differentiable particle filters: End-to-end learning with algorithmic
  priors.
\newblock In \emph{Proceedings of Robotics: Science and Systems}, 2018.

\bibitem[Kantas et~al.(2015)Kantas, Doucet, Singh, Maciejowski, and
  Chopin]{kantas2015particle}
Kantas, N., Doucet, A., Singh, S.~S., Maciejowski, J., and Chopin, N.
\newblock On particle methods for parameter estimation in state-space models.
\newblock \emph{Statistical Science}, 30\penalty0 (3):\penalty0 328--351, 2015.

\bibitem[Karkus et~al.(2018)Karkus, Hsu, and Lee]{karkus2018particle}
Karkus, P., Hsu, D., and Lee, W.~S.
\newblock Particle filter networks with application to visual localization.
\newblock In \emph{Conference on Robot Learning}, 2018.

\bibitem[Kingma \& Welling(2014)Kingma and Welling]{kingma2013auto}
Kingma, D.~P. and Welling, M.
\newblock Auto-encoding variational {B}ayes.
\newblock In \emph{International Conference on Learning Representations}, 2014.

\bibitem[Kitagawa \& Gersch(1996)Kitagawa and Gersch]{kitagawa1996smoothness}
Kitagawa, G. and Gersch, W.
\newblock \emph{Smoothness Priors Analysis of Time Series}, volume 116.
\newblock Springer Science \& Business Media, 1996.

\bibitem[Klaas et~al.(2005)Klaas, De~Freitas, and Doucet]{klaas2012toward}
Klaas, M., De~Freitas, N., and Doucet, A.
\newblock Toward practical ${N}^2$ {M}onte {C}arlo: the marginal particle
  filter.
\newblock \emph{Uncertainty in Artificial Intelligence}, 2005.

\bibitem[Kloss et~al.(2020)Kloss, Martius, and Bohg]{kloss2020train}
Kloss, A., Martius, G., and Bohg, J.
\newblock How to train your differentiable filter.
\newblock \emph{arXiv preprint arXiv:2012.14313}, 2020.

\bibitem[Krishnan et~al.(2017)Krishnan, Shalit, and
  Sontag]{krishnan2017structured}
Krishnan, R.~G., Shalit, U., and Sontag, D.
\newblock Structured inference networks for nonlinear state space models.
\newblock In \emph{AAAI Conference on Artificial Intelligence}, pp.\
  2101--2109, 2017.

\bibitem[Le et~al.(2018)Le, Igl, Rainforth, Jin, and Wood]{le2017auto}
Le, T.~A., Igl, M., Rainforth, T., Jin, T., and Wood, F.
\newblock Auto-encoding sequential {M}onte {C}arlo.
\newblock In \emph{International Conference on Learning Representations}, 2018.

\bibitem[Lee(2008)]{lee2008towards}
Lee, A.
\newblock Towards smooth particle filters for likelihood estimation with
  multivariate latent variables.
\newblock Master's thesis, University of British Columbia, 2008.

\bibitem[Li \& Nochetto(2021)Li and Nochetto]{li2020quantitative}
Li, W. and Nochetto, R.~H.
\newblock Quantitative stability and error estimates for optimal transport
  plans.
\newblock \emph{IMA Journal of Numerical Analysis}, 2021.

\bibitem[Lindsten \& Sch{\"o}n(2013)Lindsten and
  Sch{\"o}n]{lindsten2013backward}
Lindsten, F. and Sch{\"o}n, T.~B.
\newblock Backward simulation methods for {M}onte {C}arlo statistical
  inference.
\newblock \emph{Foundations and Trends{\textregistered} in Machine Learning},
  6\penalty0 (1):\penalty0 1--143, 2013.

\bibitem[Ma et~al.(2020{\natexlab{a}})Ma, Karkus, Hsu, and Lee]{ma2019particle}
Ma, X., Karkus, P., Hsu, D., and Lee, W.~S.
\newblock Particle filter recurrent neural networks.
\newblock In \emph{AAAI Conference on Artificial Intelligence},
  2020{\natexlab{a}}.

\bibitem[Ma et~al.(2020{\natexlab{b}})Ma, Karkus, Ye, Hsu, and
  Lee]{ma2020discriminative}
Ma, X., Karkus, P., Ye, N., Hsu, D., and Lee, W.~S.
\newblock Discriminative particle filter reinforcement learning for complex
  partial observations.
\newblock In \emph{International Conference on Learning Representations},
  2020{\natexlab{b}}.

\bibitem[Maddison et~al.(2017)Maddison, Lawson, Tucker, Heess, Norouzi, Mnih,
  Doucet, and Teh]{maddison2017filtering}
Maddison, C.~J., Lawson, D., Tucker, G., Heess, N., Norouzi, M., Mnih, A.,
  Doucet, A., and Teh, Y.~W.
\newblock Filtering variational objectives.
\newblock In \emph{Advances in Neural Information Processing Systems}, 2017.

\bibitem[Malik \& Pitt(2011)Malik and Pitt]{malikpitt2011particle}
Malik, S. and Pitt, M.~K.
\newblock Particle filters for continuous likelihood evaluation and
  maximisation.
\newblock \emph{Journal of Econometrics}, 165\penalty0 (2):\penalty0 190--209,
  2011.

\bibitem[Murray et~al.(2013)Murray, Jones, and Parslow]{murray2013disturbance}
Murray, L.~M., Jones, E.~M., and Parslow, J.
\newblock On disturbance state-space models and the particle marginal
  {M}etropolis--{H}astings sampler.
\newblock \emph{SIAM/ASA Journal on Uncertainty Quantification}, 1\penalty0
  (1):\penalty0 494--521, 2013.

\bibitem[Myers et~al.(2021)Myers, Thiery, Wang, and
  Bui-Thanh]{myers2019sequential}
Myers, A., Thiery, A.~H., Wang, K., and Bui-Thanh, T.
\newblock Sequential ensemble transform for {B}ayesian inverse problems.
\newblock \emph{Journal of Computational Physics}, 427:\penalty0 110055, 2021.

\bibitem[Naesseth et~al.(2018)Naesseth, Linderman, Ranganath, and
  Blei]{naesseth2017variational}
Naesseth, C.~A., Linderman, S.~W., Ranganath, R., and Blei, D.~M.
\newblock Variational sequential {M}onte {C}arlo.
\newblock In \emph{International Conference on Artificial Intelligence and
  Statistics}, 2018.

\bibitem[Peyr{\'e} \& Cuturi(2019)Peyr{\'e} and Cuturi]{peyr2019computational}
Peyr{\'e}, G. and Cuturi, M.
\newblock Computational optimal transport.
\newblock \emph{Foundations and Trends{\textregistered} in Machine Learning},
  11\penalty0 (5-6):\penalty0 355--607, 2019.

\bibitem[Poyiadjis et~al.(2011)Poyiadjis, Doucet, and
  Singh]{poyiadjis2011particle}
Poyiadjis, G., Doucet, A., and Singh, S.~S.
\newblock Particle approximations of the score and observed information matrix
  in state space models with application to parameter estimation.
\newblock \emph{Biometrika}, 98\penalty0 (1):\penalty0 65--80, 2011.

\bibitem[Rangapuram et~al.(2018)Rangapuram, Seeger, Gasthaus, Stella, Wang, and
  Januschowski]{rangapuram2018deep}
Rangapuram, S.~S., Seeger, M.~W., Gasthaus, J., Stella, L., Wang, Y., and
  Januschowski, T.
\newblock Deep state space models for time series forecasting.
\newblock In \emph{Advances in Neural Information Processing Systems}, pp.\
  7785--7794, 2018.

\bibitem[Reich(2013)]{reich2012nonparametric}
Reich, S.
\newblock A nonparametric ensemble transform method for {B}ayesian inference.
\newblock \emph{SIAM Journal on Scientific Computing}, 35\penalty0
  (4):\penalty0 A2013--A2024, 2013.

\bibitem[Scetbon \& Cuturi(2020)Scetbon and Cuturi]{scetbon2020linear}
Scetbon, M. and Cuturi, M.
\newblock Linear time {S}inkhorn divergences using positive features.
\newblock In \emph{Advances in Neural Information Processing Systems}, 2020.

\bibitem[Seguy et~al.(2018)Seguy, Damodaran, Flamary, Courty, Rolet, and
  Blondel]{seguy2017large}
Seguy, V., Damodaran, B.~B., Flamary, R., Courty, N., Rolet, A., and Blondel,
  M.
\newblock Large-scale optimal transport and mapping estimation.
\newblock In \emph{International Conference on Learning Representations}, 2018.

\bibitem[Thrun et~al.(2005)Thrun, Burgard, and Fox]{thrun2005probabilistic}
Thrun, S., Burgard, W., and Fox, D.
\newblock \emph{Probabilistic Robotics}.
\newblock MIT Press, 2005.

\bibitem[van~de Meent et~al.(2018)van~de Meent, Paige, Hongseok, and
  Wood]{vandemeent2018introduction}
van~de Meent, J.-W., Paige, B., Hongseok, Y., and Wood, F.
\newblock An introduction to probabilistic programming.
\newblock \emph{arXiv preprint arXiv:1809.10756}, 2018.

\bibitem[Villani(2008)]{villani2008optimal}
Villani, C.
\newblock \emph{Optimal Transport: Old and New}, volume 338.
\newblock Springer Science \& Business Media, 2008.

\bibitem[Weed(2018)]{weed2018explicit}
Weed, J.
\newblock An explicit analysis of the entropic penalty in linear programming.
\newblock In \emph{Proceedings of the 31st Conference On Learning Theory},
  2018.

\bibitem[Wen et~al.(2020)Wen, Chen, Papagiannis, Hu, and Li]{wen2020end}
Wen, H., Chen, X., Papagiannis, G., Hu, C., and Li, Y.
\newblock End-to-end semi-supervised learning for differentiable particle
  filters.
\newblock \emph{arXiv preprint arXiv:2011.05748}, 2020.

\bibitem[Williams(1992)]{williams1992simple}
Williams, R.~J.
\newblock Simple statistical gradient-following algorithms for connectionist
  reinforcement learning.
\newblock \emph{Machine Learning}, 8\penalty0 (3-4):\penalty0 229--256, 1992.

\bibitem[Zhu et~al.(2020)Zhu, Murphy, and Jonschkowski]{zhu2020towards}
Zhu, M., Murphy, K., and Jonschkowski, R.
\newblock Towards differentiable resampling.
\newblock \emph{arXiv preprint arXiv:2004.11938}, 2020.

\end{thebibliography}
\bibliographystyle{icml2021}

\newpage

\onecolumn

\appendix

\section{Proof of Proposition \ref{prop:asymptoticELBOgradient}}
    \label{sec:ELBOgradientbias}
     A particle filter with multinomial resampling is defined by the following joint distribution
    \begin{align*}
        \overline{q}_{\theta,\phi}(x_{1:T}^{1:N},a_{1:T-1}^{1:N}) & =\prod_{i=1}^{N}q_{\phi}\left(x_{1}^{i}\right)\prod_{t=2}^{T}\prod_{i=1}^{N}w_{t-1}^{a_{t-1}^{i}}q_{\phi}\left(x_{t}^{i}|x{}_{t-1}^{a_{t-1}^{i}}\right)
    \end{align*}
    where $a_{t-1}^{i}\in\{1,...,N\}$ is the ancestral index of particle $x_{t}^{i}$ and
    \begin{align*}
        \omega_{\theta,\phi}(x_{1},y_1) & =\frac{p_{\theta}(x_{1},y_{1})}{q_{\phi}(x_{1})},\quad \omega_{\theta,\phi}(x_{t-1},x_{t},y_t)=\frac{p_{\theta}(x_{t},y_{t}|x_{t-1})}{q_{\phi}\left(x_{t}|x_{t-1}\right)}.
    \end{align*}
    Finally, we have $w_{t}^{i}\propto \omega_{\theta,\phi}(x_{t-1}^{a_{t-1}^{i}},x_{t}^{i},y_t),~\sum_{i=1}^{N}w_{t}^{i}=1$. We do not emphasize notationally that the weights $w_{t-1}^{a_{t-1}^{i}}$ are $\theta,\phi$ and observations dependent.

    The ELBO is given by 
    \begin{align*}
        \ell^{\text{ELBO}}(\theta,\phi)= & \mathbb{E}_{\overline{q}_{\theta,\phi}}
        \left[\log\widehat{p}_{\theta}(y_{1:T})
        \right]=\mathbb{E}_{\overline{q}_{\theta,\phi}}\left[
        \log\left(
            \frac{1}{N}\sum_{i=1}^{N}\omega_{\theta,\phi}(X_{1}^{i},y_1)
        \right)
        +\sum_{t=2}^{T}\log
            \left(
            \frac{1}{N}\sum_{i=1}^{N}\omega_{\theta,\phi}
                (X_{t-1}^{A_{t-1}^{i}},X_{t}^{i},y_t)
            \right)
        \right].
    \end{align*}
    We now compute $\nabla_{\theta}\ell^{\text{ELBO}}(\theta,\phi)$. We assume from now on that the regularity conditions allowing us to swap the expectation and differentiation operators are satisfied as in \cite{maddison2017filtering,le2017auto,naesseth2017variational}. We can split the gradient using the product rule and apply the log-derivative trick:

    \begin{align}
    \nabla_{\theta}\ell^{\text{ELBO}}(\theta,\phi)  =& \mathbb{E}_{\overline{q}_{\theta,\phi}}\left[\nabla_{\theta}\log\widehat{p}_{\theta}(y_{1:T})\right]
     +\mathbb{E}_{\overline{q}_{\theta,\phi}}\left[ \log\widehat{p}_{\theta}(y_{1:T})  \nabla_{\theta}\log\overline{q}_{\theta,\phi}(X_{1:T}^{1:N},A_{1:T-1}^{1:N}) \right]\nonumber\\
      =& \mathbb{E}_{\overline{q}_{\theta,\phi}}\left[\nabla_{\theta}\log\left(\frac{1}{N}\sum_{i=1}^{N}\omega_{\theta,\phi}(X_{1}^{i},y_1)\right)+\sum_{t=2}^{T}\nabla_{\theta}\log\left(\frac{1}{N}\sum_{i=1}^{N}\omega_{\theta,\phi}(X_{t-1}^{A_{t-1}^{i}},X_{t}^{i},y_t)\right)\right]\label{eq:first_part_grad}\\
     & +\mathbb{E}_{\overline{q}_{\theta,\phi}}\left[ \log\widehat{p}_{\theta}(y_{1:T}) \left\{ \sum_{t=2}^{T}\sum_{i=1}^{N}\nabla_{\theta}\log w_{t-1}^{A_{t-1}^{i}}\right\} \right] \label{eq:second_part_grad}
    \end{align}

    For the first part of the ELBO gradient (\ref{eq:first_part_grad}), we have
    
    \begin{align*}
        \nabla_{\theta}\log\left(\frac{1}{N}\sum_{i=1}^{N} \omega_{\theta,\phi}(X_{1}^{i}),y_1\right) & =\sum_{i=1}^{N} w_{1}^{i}\nabla_{\theta}\log w_{\theta,\phi}(X_{1}^{i},y_1)=\sum_{i=1}^{N}w_{1}^{i}\nabla_{\theta}\log p_{\theta}(X_{1}^{i},y_{1})
    \end{align*}
    and
    \begin{align*}
        \nabla_{\theta}\log\left(\frac{1}{N}\sum_{i=1}^{N}\omega_{\theta,\phi}(X_{t-1}^{A_{t-1}^{i}},X_{t}^{i},y_t)\right) & =\sum_{i=1}^{N}w_{t}^{i}\nabla_{\theta}\log \omega_{\theta,\phi}(X_{t-1}^{A_{t-1}^{i}},X_{t}^{i},y_t)=\sum_{i=1}^{N}w_{t}^{i}\nabla_{\theta}\log p_{\theta}(X_{t}^{i},y_{t}|X_{t-1}^{A_{t-1}^{i}}).
    \end{align*}
    This gives
    \begin{align}
        \hspace{-2cm}\nabla_{\theta}\ell^{\text{ELBO}}(\theta,\phi) &=\mathbb{E}_{\overline{q}_{\theta,\phi}}\left[\sum_{i=1}^{N}w_{1}^{i}\nabla_{\theta}\log p_{\theta}(X_{1}^{i},y_{1})+\sum_{t=2}^{T}\sum_{i=1}^{N}w_{t}^{i}\nabla_{\theta}\log p_{\theta}(X_{t}^{i},y_{t}|X_{t-1}^{A_{t-1}^{i}})\right]\label{eq:ELBO1}\\
         & +\mathbb{E}_{\overline{q}_{\theta,\phi}}\left[ \log\widehat{p}_{\theta}(y_{1:T}) \left\{ \sum_{t=2}^{T}\sum_{i=1}^{N}\nabla_{\theta}\log w_{t-1}^{A_{t-1}^{i}}\right\} \right].\label{eq:ELBO2}
    \end{align}
    When we ignore the gradient terms due to resampling corresponding to \eqref{eq:ELBO2} as proposed in \cite{naesseth2017variational,le2017auto,maddison2017filtering,HirtDellaportas2019}, we only use an unbiased estimate of the first term  \eqref{eq:ELBO1}, i.e.
    \begin{equation}\label{eq:biasedgradientELBO}
        \hat{\nabla}_{\theta}\ell^{\textup{ELBO}}(\theta,\phi)\coloneqq\sum_{i=1}^{N}w_{1}^{i}\nabla_{\theta}\log p_{\theta}(X_{1}^{i},y_{1})+\sum_{t=2}^{T}\sum_{i=1}^{N}w_{t}^{i}\nabla_{\theta}\log p_{\theta}(X_{t}^{i},y_{t}|X_{t-1}^{A_{t-1}^{i}}),\quad\text{where~} (X^{1:N}_{1:T}, A^{1:N}_{1:T-1})\sim \overline{q}_{\theta,\phi}(\cdot).
    \end{equation}
    Now we assume that the mild assumptions ensuring almost sure convergence of the PF estimates are satisfied (see e.g. \cite{delmoral2004}). Under these assumptions, the estimator (\ref{eq:biasedgradientELBO}) converges almost surely as $N\rightarrow\infty$ towards
    \begin{align}\label{eq:asymptoticbiasgradient}
        \int & \nabla_{\theta}\log p_{\theta}(x_{1},y_{1})p_{\theta}(x_{1}|y_{1})\mathrm{d}x_{1}+\sum_{t=2}^{T}\int\nabla_{\theta}\log p_{\theta}(x_{t},y_{t}|x_{t-1})p_{\theta}(x_{t-1:t}|y_{1:t-1})\mathrm{d}x_{t-1:t}.
    \end{align}
    Under an additional uniform integrability condition on $\hat{\nabla}_{\theta}\ell^{\textup{ELBO}}(\theta,\phi)$, we thus have that  $\mathbb{E}_{\overline{q}_{\theta,\phi}}[\hat{\nabla}_{\theta}\ell^{\textup{ELBO}}(\theta,\phi)]$ converges towards \eqref{eq:asymptoticbiasgradient}. We recall that the true score is given by Fisher's identity and satisfies
    \begin{align*}
        \int & \nabla_{\theta}\log p_{\theta}(x_{1},y_{1})p_{\theta}(x_{1}|y_{1:T})\mathrm{d}x_{1}+\sum_{t=2}^{T}\int\nabla_{\theta}\log p_{\theta}(x_{t},y_{t}|x_{t-1})p_{\theta}(x_{t-1:t}|y_{1:T})\mathrm{d}x_{t-1:t}.
    \end{align*}
    This concludes the proof of Proposition \ref{prop:asymptoticELBOgradient}.

\section{Notation and Assumptions}
\label{subsec:notation}
        
        \subsection{Filtering Notation}
        \label{subsec:filteringnotation}
            Recall $\mathcal{X}=\mathbb{R}^{d_x}$, denote the Borel sets of $\mathcal{X}$ by $\mathcal{B}(\mathcal{X})$ and $\mathcal{P}(\mathcal{X})$ the set of Borel probability measures on $(\mathcal{X},\mathcal{B}(\mathcal{X}))$. In an abuse of notation, we shall use the same notation for a probability measure and its density w.r.t. Lebesgue measure; i.e. $\nu(\mathrm{d}x)=\nu(x)\mathrm{d}x$. We also use the standard notation $\nu(\psi)=\int \psi(x) \nu(x) \mathrm{d}x$ for any test function $\psi$. In the interest of notational clarity, we will remove subscript $\theta,\phi$ where unnecessary in further workings.

            \noindent We denote $\{\alpha^{(t)}\}_{t\geq 0}$ the predictive distributions where $\alpha^{(t)}(x_t) =p(x_{t}|y_{1:{t-1}})$ for $t> 1$ and $\alpha^{(1)}(x_1) = \mu(x_1)$ while $\{\beta^{(t)}\}_{t\geq 1}$ denotes the filtering distributions; i.e.  $\beta^{(t)}(x_t) =p(x_{t}|y_{1:t})$ for $t\geq 1$.\\

          Using this notation, we have 
                \begin{align}
                \label{eq:recurs1}
                 &\alpha^{(t)}(\psi)=\int \psi(x_t) f(x_{t}|x_{t-1})\beta^{(t-1)}(x_{t-1})\mathrm{d}x_{t-1}\mathrm{d}x_{t}\coloneqq\beta^{(t-1)}f(\psi),\\
                 &\beta^{(t)}(\psi) = \frac{ \alpha^{(t)}(g(y_t|\cdot) \psi)}{ \alpha^{(t)}(g(y_t|\cdot))}= \frac{ \beta^{(t-1)}(f(g(y_t|\cdot) \psi))}{ \beta^{(t-1)}(f(g(y_t|\cdot)))}.
                \end{align}
            More generally, for a proposal distribution $q(x_t|x_{t-1},y_t)\neq f(x_{t}|x_{t-1})$ with parameter $\phi\neq\theta$, the following recursion holds
                
                \begin{align}
                \label{eq:filteringoperator}
                 &\beta^{(t)}(\psi) = \frac{ \beta^{(t-1)}(q(\omega_{t} \, \psi))}{ \beta^{(t-1)}(q(\omega_{t}))}
                 \end{align}
                
                \begin{align}
                    \omega_{t}(x_{t-1},x_t)\coloneqq\omega(x_{t-1},x_t,y_t)=\frac{g(y_t|x_t)f(x_t|x_{t-1})}{q(x_t|x_{t-1},y_t)}.
                \end{align}
             To simplify the presentation, we will present the analysis in the scenario where $\phi=\theta$ and $q(x_t|x_{t-1},y_t)=f(x_t|x_{t-1})$ so we will analyze \eqref{eq:recurs1} for which $\omega_{t}(x_{t-1},x_t)=g(y_t|x_t)$. In this case, the particle approximations of $\mu$ is denoted $\mu_N$ and for $t>1$, $\alpha^{(t)}$ and $\beta^{(t)}$ are given by the random measures
            \begin{equation}
            		\alpha^{(t)}_{N}(\psi) = \frac{1}{N} \sum_{i=1}^N\psi(X^{i}_t),\quad \beta_N^{(t)}(\psi) =\sum_{i=1}^N w^{i}_t\psi(X_t^{i}),\quad \tilde{\beta}_{N}^{(t)}(\psi) = \frac{1}{N} \sum_{i=1}^N\psi(\tilde{X}_t^{i}),
            \end{equation}
            where $w^{i}_t \propto g(y_t|X_t^{i})$ with $\sum_{i=1}^N w^{i}_t=1$ and particles are drawn from $X^{i}_t \sim f(\cdot|\tilde{X}^{i}_{t-1})$.

            Here $\beta_N^{(t)}$ denotes the weighted particle approximation of $\beta^{(t)}$ while  $\tilde{\beta}_{N}^{(t)}$ is the uniformly weighted approximation obtained after the DET transformation described in Section \ref{sec:differentiable_et}.

        \subsection{Optimal Transport Notation}
        
        Recall from Section \ref{sec:optimal_transport}, $\mathcal{P}_t^{\text{OT}}$ denotes a transport between $\alpha^{(t)}$ and $\beta^{(t)}$ with accompanying map $\mathbf{T}^{(t)}$. $\mathcal{P}_t^{\text{OT},N}$ denotes an optimal transport between particle approximations $\alpha_N^{(t)}$ and $\beta_N^{(t)}$ with corresponding transport matrix, $\mathbf{P}^{\text{OT}}$ with $i,j$ entry $p^\text{OT}_{i,j}$. To simplify notation, we remove script $t$ when not needed.

        Similarly from Section \ref{sec:diff_ot}, $\mathcal{P}^{\text{OT,N}}_\epsilon$ denotes the regularized transport between $\alpha^{(t)}_N$ and $\beta^{(t)}_N$ with accompanying matrix $\mathbf{P}^{\text{OT}}_\epsilon$ with $i,j$ entry $p^\text{OT}_{\epsilon,i,j}$. Recall $\tilde{\beta}^{(t)}_N=\frac{1}{N} \sum_{i=1}^N\delta_{\tilde{X}^i}$ is the uniformly weighted particle approximation for $\beta^{(t)}$ under the DET, i.e. $\tilde{X}^{i}=\mathbf{T}^{(t)}_{N, \epsilon}(X^i)=\int y\mathcal{P}^{\mathrm{OT},N}_\epsilon(\mathrm{d}y|x^i)$. Note that $\tilde{X}^{i}_{N,\epsilon}$ will be used where necessary to avoid ambiguity when comparing to other resampling schemes.
        
        Recall also for $p>0$:
        \begin{align}\label{eq:Wasserstein_p}
            \wass_p^{p}(\alpha, \beta) &=\min_{\mathcal{P}\in \mathcal{U}(\alpha,\beta)} \mathbb{E}_{(U,V) \sim \mathcal{P}}\big[||U-V||^p\big]
        \end{align} where $\mathcal{U}(\alpha,\beta)$ is the collection of couplings with marginals $\alpha$ and $\beta$.
    
\subsection{Assumptions}
Our results will rely on the following four assumptions.

\begin{assumption}\label{ass:compact}
                $\mathcal{X}\subset \mathbb{R}^d$ is a compact subset with diameter
                $$\mathfrak{d}\coloneqq \sup_{x,y\in \mathcal{X}} |x-y|.$$ 
            \end{assumption}
\vspace{0.5cm}
\begin{assumption}\label{ass:lipcontraction}
                There exists $\kappa\in(0,1)$ such that for any two probability measures $\pi, \rho$ on $\mathcal{X}$ 
                $$\wass_k(\pi f, \rho f) \leq \kappa \wass_k(\pi, \rho), \qquad k=1,2.$$
            \end{assumption}

\vspace{0.5cm}        
\begin{assumption}\label{ass:omega}
                The weight function $\omega^{(t)}:\mathcal{X}\to [\Delta, \Delta^{-1}]$ is 1-Lipschitz for all $t$. 
            \end{assumption}

\vspace{0.5cm}        
\begin{assumption}\label{ass:lipschitz}
                There exists a $\lambda >0$, such that for all $t\geq 0$ the unique optimal transport plan between $\alpha^{(t)}$ and $\beta^{(t)}$ is given by a deterministic, $\lambda$-Lipschitz map $\mathbf{T}^{(t)}$. 
        \end{assumption}


\section{Auxiliary Results and Proof of Proposition \ref{prop:CVDETDetailed} 
} 
We start by establishing a couple of key auxiliary results which will be then used subsequently to establish Proposition \ref{prop:CVDETDetailed}.

\subsection{Auxiliary Results}
As per section \ref{sec:optimal_transport}, let $\mathcal{S}(\alpha_N, \beta_N)$ denote the collection of coupling matrices between $\alpha_N=\sum_{i=1}^N a_i \delta_{Y^i}$ with $a_i>0$ and $\beta_N=\sum_{i=1}^N b_i \delta_{X^i}$. We also denote entropy by $H$ where  $H(\mathbf{P}) = \sum_{i,j} p_{i,j} \log(1/p_{i,j})$ for $\mathbf{P}=(p_{i,j})_{i,j} \in \mathcal{S}(\alpha_N, \beta_N)$.
\begin{lemma}\label{lem:entropic_radius}
    The entropic radius, $R_H$, of simplex $\mathcal{U}(\alpha_N, \beta_N)$ may be bounded above as follows
    $$R_H \coloneqq \max_{\mathbf{P}_1, \mathbf{P}_2\in \mathcal{S}(\alpha_N, \beta_N)} H(\mathbf{P}_1) - H(\mathbf{P}_2) \leq 2\log (N)$$
\end{lemma}
\begin{proof}
    Notice that $-H(\mathbf{P})$ is convex, so that $H(\mathbf{P})$ is concave.
\begin{align*}
    \sum_{i,j} p_{i,j} \log\left(\frac{1}{p_{i,j}}\right)
    &=N^2\sum_{i,j} \frac{1}{N^2}p_{i,j} \log\left(\frac{1}{p_{i,j}}\right)\\
    &\leq N^2 H\left(\frac{1}{N^2} \sum_{i,j} p_{i,j}\right) = N^2 H(1/N^2)=N^2 \frac{1}{N^2} \log(N^2)= 2\log (N).
\end{align*}
In addition since $p_{i,j} \leq 1$ for all $i,j$, we have that $H(\mathbf{P})\geq 0$ and therefore we can bound 
$$R_H=\max_{P_1, P_2\in \mathcal{P}} H(\mathbf{P}_1)-H(\mathbf{P}_2) \leq \max_{\mathbf{P}_1\in \mathcal{S}(\alpha_N, \beta_N)} H(\mathbf{P}_1)\leq 2\log (N).$$
\end{proof}


\begin{lemma}
\label{lem:linochetto}
    Let $\mathcal{X}\subset \mathbb{R}^d$ be compact with diameter $\mathfrak{d}>0$. 
    Suppose we are given two probability measures $\alpha, \beta$ on $\mathcal{X}$ with a unique deterministic, $\lambda$-Lipschitz optimal transport map $\mathbf{T}$ while $\alpha_N=\sum_{i=1}^N a_i \delta_{Y^i}$ with $a_i>0$ and $\beta_N=\sum_{i=1}^N b_i \delta_{X^i}$. We write $\mathcal{P}^{\mathrm{OT},N}$, resp. $\mathcal{P}^{\mathrm{OT},N}_\epsilon$, for an optimal coupling between $\alpha_N$ and $\beta_N$, resp. the $\epsilon$-regularized optimal transport plan, between $\alpha_N$ and $\beta_N$.
    Then 
    \begin{align*}
        \left[\int ||y - \mathbf{T}(x)||^2 \mathcal{P}^{\mathrm{OT},N}_{\epsilon}(\rd x, \rd y) \right]^{\frac{1}{2}} 
        \leq 2\lambda^{1/2} 
        \mathcal{E}^{1/2}\left[\mathfrak{d}^{1/2}+ \mathcal{E}\right]^{1/2}+ \max\{\lambda,  1\}\left[ \wass_2(\alpha_N, \alpha) + \wass_2(\beta_N, \beta)\right],
    \end{align*}
    where
    $$\mathcal{E}\coloneqq \mathcal{E}(N, \epsilon, \alpha, \beta)\coloneqq \wass_2(\alpha_N, \alpha)+\wass_2(\beta_N, \beta) + \sqrt{2\epsilon \log (N)}.$$
\end{lemma}
\begin{proof}
    From Corollary 3.8 from \cite{li2020quantitative} 
    \begin{align*}
        \left[\int \|\mathbf{T}(x)-y\|^2 \pne (\rd x, \rd y) \right]^{1/2}
        &\leq 2\lambda^{1/2} 
        \sqrt{\tilde{e}_{N,\epsilon}}\left[\wass_2(\alpha, \beta)+\tilde{e}_{N,\epsilon} \right]^{1/2}+ \lambda \wass_2(\alpha_N, \alpha) + \wass_2(\beta_N, \beta),
    \end{align*}
    where $\lambda$ is the Lipschitz constant of the optimal transport map $\mathbf{T}$ sending $\alpha$ to $\beta$, and
    \begin{equation}
        \tilde{e}_{N, \epsilon}\coloneqq \wass_2(\alpha_N, \alpha) + \wass_2(\beta_N, \beta) + 
        \left[\int \|x-y\|^2 \pne(\rd x, \rd y) \right]^{1/2} - \wass_2(\alpha_N, \beta_N).
    \end{equation}
    From Proposition~4 of \cite{weed2018explicit},  $$\sum_{i,j=1,\dots, N} p^\mathrm{OT}_{\epsilon,i,j} |Y_i-X_j|^2 - \wass^2_2(\alpha_N, \beta_N)\leq \epsilon R_H,$$ 
    where $R_H$ is the entropic radius as defined in Lemma \ref{lem:entropic_radius}.

  
    By Lemma \ref{lem:entropic_radius} we therefore have that  $$\int \|x-y\|^2 \pne(\rd x, \rd y)  - \wass^2_2(\alpha_N, \beta_N)
    \leq 2\epsilon \log (N).$$
    Since $x\mapsto \sqrt{x}$ is sub-additive, for $r,s >0$ we have that 
    $\sqrt{r}-\sqrt{s}\leq \sqrt{r-s}$, whence 
    $$\left[\int \|x-y\|^2 \pne(\rd x, \rd y) \right]^{1/2} -\wass_2(\alpha_N, \beta_N)
    \leq \sqrt{2\epsilon \log N}. $$
    We thus have 
    $$\tilde{e}_{N,\epsilon} \leq 
    \wass_2(\alpha_N,\alpha) + \wass_2(\beta_N, \beta) + \sqrt{2\epsilon \log(N)}.$$
    In addition, by Assumption~\ref{ass:compact} we have that 
    $\wass_2(\alpha, \beta)\leq \mathfrak{d}^{1/2}$ and the result follows.  
\end{proof}

\subsection{Proof of Proposition \ref{prop:CVDETDetailed}}
\label{sec:CVDET_proof}
    \begin{proof}[Proof of Proposition \ref{prop:CVDETDetailed}]
By definition, we have $\tilde{\beta}_N(\mathrm{d}\tilde{x})=\int \alpha_N(\mathrm{d}x)  \delta_{\mathbf{T}_{N,\epsilon}(x)}(\mathrm{d}\tilde{x})$ with $\mathbf{T}_{N,\epsilon}(x)\coloneqq\int \tilde{x} \mathcal{P}^{\mathrm{OT,N}}_\epsilon(\mathrm{d}\tilde{x}|x)$ while, as $\mathcal{P}^{\mathrm{OT},N}_\epsilon$ belongs to $\mathcal{U}(\alpha_N,\beta_N)$, we also have $\beta_N(\mathrm{d}\tilde{x})=\int \alpha_N(\mathrm{d}x) \mathcal{P}^{\mathrm{OT,N}}_\epsilon(\mathrm{d}\tilde{x}|x)$. We then have for any 1-Lipschitz function
        \begin{align*}
                    \left|\beta_N(\psi)-\tilde{\beta}_N(\psi)\right|
                    &=\left|\int \left[\int (\psi(\tilde{x})-\psi(\mathbf{T}_{N,\epsilon}(x))) \mathcal{P}^{\mathrm{OT,N}}_\epsilon(\mathrm{d}\tilde{x}|x)\right] \alpha_N(\mathrm{d}x) \right|\\
                    &\leq \iint \left|\psi(\tilde{x})-\psi(\mathbf{T}_{N,\epsilon}(x))\right|\alpha_N(\mathrm{d}x) \mathcal{P}^{\mathrm{OT,N}}_\epsilon(\mathrm{d}\tilde{x}|x) \\
                    &\leq \iint ||\tilde{x}-\mathbf{T}_{N,\epsilon}(x)||  \mathcal{P}^{\mathrm{OT,N}}_\epsilon(\mathrm{d}x,\mathrm{d}\tilde{x}) \\
                   &\leq \left(\iint||\tilde{x}-\mathbf{T}_{N, \epsilon}(x)||^2\mathcal{P}^{\mathrm{OT,N}}_\epsilon(\mathrm{d}x,\mathrm{d}\tilde{x})\right)^{\frac 1 2}\\
                   &\leq \left(\iint||\tilde{x}-\mathbf{T}(x)||^2\mathcal{P}^{\mathrm{OT,N}}_\epsilon(\mathrm{d}x,\mathrm{d}\tilde{x})\right)^{\frac 1 2},
        \end{align*}
        where the final inequality follows from the fact that for any random vector $V$ the mapping $v\mapsto\mathbb{E}[\|V-v\|^2]$ is minimized at $v=\mathbb{E}[V]$. 
        The stated result is then obtained using Lemma \ref{lem:linochetto}.
    \end{proof}

\section{Proof of Proposition \ref{prop:bias}}\label{sec:ProofconsistencyDPF}
For technical reasons, we analyse here a slightly modified PF algorithm where 
\begin{equation}\label{eq:newalgo}
\alpha_N^{(t)} = \frac{1}{N} \sum_{j=1}^N \delta_{X^j_t}, 
\qquad  X^j_t \stackrel{\textrm{i.i.d.}}{\sim} \tilde{\beta}_{N}^{(t-1)}f
= \frac{1}{N} \sum_{j=1}^N f\left(\cdot \Big| \tilde{X}^j_{t-1}\right).
\end{equation}
instead of the standard version where one has 
$$\alpha_N^{(t)} = \frac{1}{N} \sum_{j=1}^N \delta_{X^j_t}, 
\qquad X^j_t \sim f\left(\cdot \Big| \tilde{X}^j_{t-1}\right).$$
This slightly modified version of the bootstrap PF was analyzed for example in \cite{del2001stability}.
The analysis does capture the additional error arising from the use of DET instead of resampling. Similar results should hold for the standard PF algorithm.
The main technical reason for analysing this modified algorithm is our reliance on Theorem 2 of \cite{fournier2015rate}; analysing the standard PF algorithm  requires a version of  \cite{fournier2015rate} for stratified sampling and will be done in future work.

\begin{proposition}\label{prop:recursivebound}
Suppose that Assumptions~\ref{ass:compact}, \ref{ass:lipcontraction} and \ref{ass:omega} hold. Suppose also that given $\tilde{\beta}_N^{(t-1)}$, 
$\alpha_N^{(t)}$ is defined through \eqref{eq:newalgo}. 
Define the functions
\begin{align}
\mathcal{F}(x)
&\coloneqq x+\sqrt{\mathfrak{d} K_1(\Delta, \mathfrak{d}) x} \nonumber\\
f_d(x)&\coloneqq \begin{cases}
			x, & d<4\\
			\frac{x}{\log(2+1/x)}, & d=4 \nonumber\\
			x^{d/2}, & d>4.
			\end{cases}\\
\mathcal{F}_{N, \epsilon, \delta, d}\left( x \right)
&\coloneqq \mathcal{F}\left(\kappa x + \sqrt{f_d^{-1} \left(\frac{\log(C/\delta)}{cN} \right)}\right),\nonumber\\
\frac{1}{\mathfrak{d}}\mathfrak{G}^2_{\epsilon, \delta, N, d} (x)
&\coloneqq 2\lambda^{1/2} 
\left[\mathcal{F}_{N, \epsilon, \delta, d}\left( x \right) + \sqrt{2\epsilon \log N} \right]^{1/2}
\left[ \mathfrak{d}^{1/2}+ \mathcal{F}_{N, \epsilon, \delta, d}\left( x \right) + \sqrt{2\epsilon \log N}\right]^{1/2}\notag \\
&\qquad + \lambda \kappa \mathcal{F}_{N, \epsilon, \delta, d}\left( x \right) + \max\{\lambda,  1\}\mathcal{F}_{N, \epsilon, \delta, d}\left( x \right).\label{eq:bigGdef}
\end{align}
%
Then for any $\epsilon, \delta>0$ we have with probability at least $1-\delta$, over the sampling step in \eqref{eq:newalgo}, that 
\begin{equation}\label{eq:recursivebound}
\wass_2\left(\tilde{\beta}_N^{(t)}, \beta^{(t)}\right)
\leq \mathfrak{G}_{\epsilon, \delta, N, d} \left[\wass_2\left(\tilde{\beta}_N^{(t-1)}, \beta^{(t-1)}\right)\right]
\end{equation}
In particular if $\wass_2(\tilde{\beta}_N^{(t-1)}, \beta^{(t-1)})\to 0$ and $\epsilon_N = o(1/\log(N))$ as $N\to \infty$ we have that $$\wass_2\big(\tilde{\beta}_N^{(t)}, \beta^{(t)}\big)
\to 0,$$
in probability.
\end{proposition}

\begin{proof}[Proof of Proposition~\ref{prop:recursivebound}]
To keep notation concise we write for $N\geq 1$
$$\alpha_N \coloneqq \alpha_N^{(t)}, \quad \alpha_N'\coloneqq\tilde{\beta}_{N}^{(t-1)}f, \quad \beta_N\coloneqq \beta_{N}^{(t)}, \quad \tilde{\beta}_{N}\coloneqq\tilde{\beta}_{N}^{(t-1)}.$$
%
%

\noindent\textbf{Controlling $\wass_1(\beta_N, \beta)$.}
Let $\psi$ be 1-Lipschitz. Without loss of generality we may assume that $\psi(0)=0$ since otherwise we can remove a constant.

\begin{align*}
\left| \beta_N (\psi) - \beta(\psi) \right|
&=\left| \frac{\alpha_N (\omega\psi)}{\alpha_N (\omega)} - \frac{\alpha (\omega\psi)}{\alpha (\omega)} \right|\\
&\leq \left| \frac{\alpha_N (\omega\psi)}{\alpha_N (\omega)} -  \frac{\alpha (\omega\psi)}{\alpha_N (\omega)}\right| +
\left|\frac{\alpha (\omega\psi)}{\alpha_N (\omega)}
- \frac{\alpha (\omega\psi)}{\alpha (\omega)} 
\right|\\
&\leq \Delta^{-1} \left| \alpha_N(\omega \psi)- \alpha(\omega \psi)\right| + \Delta^{-2}\alpha(\omega \psi) 
|\alpha_N(\omega) - \alpha(\omega)|. 
\end{align*}
At this stage notice that 
$$|(\omega \psi)'|\leq  |\omega' \psi| + |\omega \psi'| \leq \|\psi\|_\infty + \|\omega\|_\infty.$$
Notice that
$$|\psi(x)| = |\psi(x)-\psi(0)| \leq |x-0|\leq \mathfrak{d}.$$
Therefore we have that 
$$|(\omega \psi)'|\leq  \mathfrak{d} + \Delta^{-1},$$
and thus $\omega\psi$ is $(\mathfrak{d} + \Delta^{-1})$-Lipschitz. 
It follows that
\begin{align*}
\left| \beta_N (\psi) - \beta(\psi) \right|
&\leq \Delta^{-1} \left| \alpha_N(\omega \psi)- \alpha(\omega \psi)\right| + \Delta^{-2}\alpha(\omega \psi) 
|\alpha_N(\omega) - \alpha(\omega)|\\
&\leq \Delta^{-1} (\mathfrak{d} + \Delta^{-1}) \wass_1(\alpha_N, \alpha) + \Delta^{-3}\mathfrak{d}\wass_1(\alpha_N, \alpha)\\
&=: K_1(\Delta, \mathfrak{d}) \wass_1(\alpha_N, \alpha).
\end{align*}
Therefore we have that 
\begin{equation}\label{eq:betabound}
\wass_1(\beta_N, \beta) \leq K_1(\Delta, \mathfrak{d}) \wass_1(\alpha_N, \alpha).
\end{equation}
Notice that using the compactness of the state space we easily get also that 
\begin{equation}\label{eq:betabound2}
\wass_2(\beta_N, \beta) \leq \sqrt{\mathfrak{d} \wass_1(\beta_N, \beta)}
\leq 
\sqrt{\mathfrak{d} K_1(\Delta, \mathfrak{d}) \wass_1(\alpha_N, \alpha)}
\leq \sqrt{\mathfrak{d} K_1(\Delta, \mathfrak{d}) \wass_2(\alpha_N, \alpha)},
\end{equation}
since clearly $\wass_1(\rho, \sigma) \leq \wass_2(\rho, \sigma)$ for any two probability measures $\rho, \sigma$. 

\noindent\textbf{Controlling $\wass_1(\tilde{\beta}_{N, \epsilon}, {\beta})$.}
Again supposing $\psi$ is 1-Lipschitz, and $\psi(0)=0$, consider
\begin{align*}
\left| \tilde{\beta}_N(\psi)-\tilde{\beta}(\psi) \right|
&= \left|\int \psi(\mathbf{T}_{N, \epsilon}(x)) \alpha_N(\rd x) - \int \psi(\mathbf{T}(x)) \alpha(\rd x)\right|\\
&\leq \left| \int \psi(\mathbf{T}_{N, \epsilon}(x)) \alpha_N(\rd x)  - \int \psi(\mathbf{T}(x)) \alpha_N(\rd x)\right|\\
&\quad + \left| \int \psi(\mathbf{T}(x)) \alpha_N(\rd x) - \int \psi(\mathbf{T}(x)) \alpha(\rd x) \right|
\end{align*}
For the second term, using the fact that $\mathbf{T}$ and $\psi$ are $\lambda$- and 1-Lipschitz respectively, we have that $\psi\circ\mathbf{T}$ is $\lambda$-Lipschitz and therefore
\begin{align*}
\left| \int \psi(\mathbf{T}(x)) \alpha_N(\rd x) - \int \psi(\mathbf{T}(x)) \alpha(\rd x) \right|
&\leq \lambda \wass_1(\alpha_N, \alpha)\leq \lambda \wass_2(\alpha_N, \alpha),
\end{align*}
where we used  Assumption~\ref{ass:lipcontraction} for that last inequality
For the first term recall that using Cauchy-Schwarz and Jensen we get
\begin{align*}
\lefteqn{\left| \int \psi(\mathbf{T}_{N, \epsilon}(x)) \alpha_N(\rd x)  - \int \psi(\mathbf{T}(x)) \alpha_N(\rd x)\right|}\\
&\leq \int \left| \mathbf{T}_{N, \epsilon}(x)   - \mathbf{T}(x) \right|\alpha_N(\rd x)\\
&\leq \int \left| \int y\mathcal{P}_{N, \epsilon}(x, \rd y)   - \mathbf{T}(x) \right|\alpha_N(\rd x)\\
&\leq \iint \left|  y   - \mathbf{T}(x) \right|\alpha_N(\rd x)\mathcal{P}_{N, \epsilon}(x, \rd y)\\
&\leq \left[\iint \left|  y   - \mathbf{T}(x) \right|^2\alpha_N(\rd x)\mathcal{P}_{N, \epsilon}(x, \rd y)\right]^{1/2}.
\end{align*}
Here we can directly apply Lemma~\ref{lem:linochetto} to obtain
\begin{align*}
\lefteqn{\left[\iint \left|  y   - \mathbf{T}(x) \right|^2\alpha_N(\rd x)\mathcal{P}_{N, \epsilon}(x, \rd y)\right]^{1/2}}\\
&\leq 2\lambda^{1/2} 
\mathcal{E}^{1/2}\left[\mathfrak{d}^{1/2}+ \mathcal{E}\right]^{1/2}+ \max\{\lambda,  1\}\left[ \wass_2(\alpha_N, \alpha) + \wass_2(\beta_N, \beta)\right],
\end{align*}
where
$$\mathcal{E}\coloneqq \mathcal{E}(n, \epsilon, \alpha, \beta)\coloneqq \wass_2(\alpha_N, \alpha)+\wass_2(\beta_N, \beta) + \sqrt{2\epsilon \log (N)}.$$
From \eqref{eq:betabound2} we have that 
$$\wass_2(\alpha_N, \alpha) + \wass_2(\beta_N, \beta)\leq \wass_2(\alpha_N, \alpha)+\sqrt{\mathfrak{d} K_1(\Delta, \mathfrak{d}) \wass_2(\alpha_N, \alpha)}.$$
Next we want to bound $\wass_2(\alpha_N, \alpha)$. Notice first that 
$$\wass_2(\alpha_N, \alpha) \leq \wass_2(\alpha_N, \alpha_N') + \wass_2(\alpha_N', \alpha)
\leq \wass_2(\alpha_N, \alpha_N') + \kappa \wass_2\left(\tilde{\beta}_N^{(t-1)}, \beta^{(t-1)}\right), $$
by Assumption~\ref{ass:lipcontraction}. 

To control the other term we use \cite{fournier2015rate} to obtain a high probability bound on 
$\wass_2(\alpha_N, \alpha_N')$. 
In particular, using Theorem~2 from \cite{fournier2015rate}, with $\alpha=\infty$ since we are in a compact domain, that for some positive constants $C,c$ we have
\begin{equation}\label{eq:fournierbound}
\mathbb{P}\left[\wass_2^2(\alpha_N, \alpha_N') \geq x \right] \leq C\exp\left[ - c N f^2_d(x)\right], 
\end{equation}
where 
\begin{equation}
f_d(x)\coloneqq \begin{cases}
			x, & d<4\\
			\frac{x}{\log(2+1/x)}, & d=4\\
			x^{d/2}, & d>4.
			\end{cases}
\end{equation}
In particular, for any $\delta>0$, with probability at least $1-\delta$ over the sampling step in $F_N$ we have that 
\begin{equation}
\wass_2(\alpha_N, \alpha_N') \leq \sqrt{f_d^{-1} \left(\frac{\log(C/\delta)}{cN} \right)}.
\end{equation}
Assuming that $d\geq 4$ the rate then is of order $N^{-1/d}$ as expected. 

Therefore with probability at least $1-\delta$ over the sampling step we have that 
\begin{equation*}
\wass_2(\alpha_N, \alpha) + \wass_2(\beta_N, \beta) 
\leq \mathcal{F}_{N, \epsilon, \delta, d}\left(\wass_2\left(\tilde{\beta}_N^{(t-1)}, \beta^{(t-1)}\right) \right), 
\end{equation*}
where
\begin{equation}
\mathcal{F}_{N, \epsilon, \delta, d}\left( x \right)
= \mathcal{F}\left(\kappa x + \sqrt{f_d^{-1} \left(\frac{\log(C/\delta)}{cN} \right)}\right), \qquad \mathcal{F}(x)\coloneqq x+\sqrt{\mathfrak{d} K_1(\Delta, \mathfrak{d}) x}
\end{equation}

Thus overall we have with probability at least $1-\delta$ over the sample
\begin{align*}
\wass_2 (\tilde{\beta}_{N, \epsilon}, \tilde{\beta})
&\leq \sqrt{\mathfrak{d}\wass_1(\tilde{\beta}_{N, \epsilon}, \tilde{\beta})}
\leq \mathfrak{G}_{\epsilon, \delta, N, d} \left( \wass_2\left(\tilde{\beta}_N^{(t-1)}, \beta^{(t-1)}\right) \right),
\end{align*}
where 
\begin{align*}
\frac{1}{\mathfrak{d}}\mathfrak{G}^2_{\epsilon, \delta, N, d} (x)
&\coloneqq 2\lambda^{1/2} 
\left[\mathcal{F}_{N, \epsilon, \delta, d}\left( x \right) + \sqrt{2\epsilon \log N} \right]^{1/2}
\left[ \mathfrak{d}^{1/2}+ \mathcal{F}_{N, \epsilon, \delta, d}\left( x \right) + \sqrt{2\epsilon \log N}\right]^{1/2}\\
&\qquad + \lambda \kappa \mathcal{F}_{N, \epsilon, \delta, d}\left( x \right) + \max\{\lambda,  1\}\mathcal{F}_{N, \epsilon, \delta, d}\left( x \right).
\end{align*}
In particular notice that if we set $\epsilon_N = o(1/\log N)$ and $x_N=o(1)$ we have
$$\mathfrak{G}_{\epsilon_N, \delta, N, d} (x_N)\to 0.$$
Therefore, notice that if $\epsilon_N=o(1/\log N)$ and $\wass_2(\mu_N, \mu) \to 0$, then for any $x>0$ we have that
$$\mathbb{P}\left[ \wass_2\big(\tilde{\beta}_{N, \epsilon}, \tilde{\beta}\big) \geq x\right]
\leq \mathbb{P}[ \wass_2(\alpha_N', \alpha_N) \geq x'],$$
for some $x'$ that does not depend on $N$, where the probability is over the sampling step. The convergence in probability follows.
\end{proof}

 \begin{proposition}\label{prop:bias2}
Let $\mu_N=\tfrac{1}{N}\sum_{i=1}^N \delta_{X_1^i}$ where $X^{i}_1\stackrel{\mathrm{i.i.d.}}{\sim}\mu\coloneqq q(\cdot | y_1)$ for $i\in [N]$ and suppose that for $t\geq 1$,  $\alpha_N^{(t)}$ is defined through \eqref{eq:newalgo}. 
Under Assumptions \ref{ass:compact}, \ref{ass:lipcontraction}, \ref{ass:omega} and \ref{ass:lipschitz}, 
for any $\delta>0$, with probability at least $1-2\delta$ over the sampling steps, for any bounded 1-Lipschitz $\psi$, for any $t\in [1:T]$,  the approximations of the filtering distributions and log-likelihood computed by DPF satisfy
\begin{align}
                |\tilde{\beta}^{(t)}_{N}(\psi)-\beta^{(t)}(\psi)
               | &\leq  \mathfrak{G}_{\epsilon, \delta/T, N, d}^{(t)}\left(\sqrt{f_d^{-1} \left(\frac{\log(CT/\delta)}{cN} \right)}\right)
               \label{eq:timetbound}\\
\label{eq:boundbias2}
         \left|
         \log \frac{\hat{p}_N(y_{1:T})}{p(y_{1:T})}\right|  &\leq \frac{\kappa}{\Delta} \max_{t\in [1:T]}\mathrm{Lip}\left[g(y_t\mid \cdot) \right]
    \sum_{t=1}^T  \mathfrak{G}_{\epsilon, \delta/T, N, d}^{(t)}\left(\sqrt{f_d^{-1} \left(\frac{\log(CT/\delta)}{cN} \right)}\right)
            \end{align}
    where $C$ is a finite constant independent of $T$, $\mathfrak{G}_{\epsilon, \delta/T, N, d}, f_d$ are defined in \eqref{eq:bigGdef}, and $\mathrm{Lip}[f]$ is the Lipschitz constant of the function $f$. $\mathfrak{G}^{(t)}_{\epsilon, \delta/T, N, d}$ denotes the $t$-repeated composition of function $\mathfrak{G}_{\epsilon, \delta/T, N, d}$. In particular, if we set $\epsilon_N=o(1/\log N)$
    $$\left|
         \log \frac{\hat{p}_N(y_{1:T})}{p(y_{1:T})}\right| \to 0,$$
         in probability.
\end{proposition}
\begin{proof}[Proof of Proposition~\ref{prop:bias2}]
Following the proof of Proposition~\ref{prop:recursivebound}, 
we define ${\alpha_N^{(t)}}' = \tilde{\beta}_N^{(t-1)}f$ and 
 for $t\in [1:T]$, the events
$$A_t
\coloneqq \wass_2\left(\alpha_N^{(t)}, {\alpha_N^{(t)}}' \right)
\leq\sqrt{f_d^{-1} \left(\frac{\log(CT/\delta)}{cN} \right)}.
$$
We know from Theorem~2 in \cite{fournier2015rate} that $\mathbb{P}(A_t) \geq 1-\delta/T$, where the probability is over the sampling step. In particular we have that
$$\mathbb{P}\left[\bigcap_{t=1}^T A_t\right] 
= 1- \mathbb{P}\left[\bigcup_{t=1}^T A_t^\mathtt{C}\right] 
\geq 1-\sum_{t=1}^N \mathbb{P}\left[A_t^\mathtt{C} \right]\geq 1- T\frac{\delta}{T} = 1-\delta. $$
Notice that on the event $\cap_{t=1}^T A_t$,  iterating the bound \eqref{eq:recursivebound} we have 
\begin{align*}
 \wass_2\left(\tilde{\beta}_N^{(t)}, \beta^{(t)} \right)   &\leq \mathfrak{G}_{\epsilon, \delta/T, N, d}^{(t)}\left(\wass_2\left(\mu_N, \mu\right)\right),
\end{align*}
with probability at least $1-\delta$. Again by Theorem 2 in \cite{fournier2015rate} we have that with probability at least $1-\delta$ 
$$\wass_2(\mu_N, \mu)\leq \sqrt{f_d^{-1} \left(\frac{\log(CT/\delta)}{cN} \right)}.$$
Therefore with probability at least $1-2\delta$ we have
 $$\mathfrak{G}_{\epsilon, \delta/T, N, d}^{(t)}\left(\sqrt{f_d^{-1} \left(\frac{\log(CT/\delta)}{cN} \right)}\right).   
$$
 
 It remains to prove \eqref{eq:boundbias2}.   
 Note that $|\log(x)-\log(y)| \leq \frac{|x-y|}{\min\{x,y\}}$ for any $x,y>0$ so
                    \begin{align}
                    \big|\log \hat{p}(y_{1:T}) - \log p(y_{1:T})\big| &\leq \sum_{t=1}^T\big|\log \hat{p}(y_t|y_{1:t-1}) - \log p(y_t|y_{1:t-1})\big|\nonumber\\
                    &\leq \sum_{t=1}^T\big|\frac{\hat{p}(y_t|y_{1:t-1})-p(y_t|y_{1:t-1})}{\min(\hat{p}(y_t|y_{1:t-1}),p(y_t|y_{1:t-1}))}\big|\nonumber\\
                    &\leq \Delta^{-1} \sum_{t=1}^T |\hat{p}(y_t|y_{1:t-1}) - p(y_t|y_{1:t-1})\big| \label{eq:sum_ll_err}
                    \end{align} 
                    where $\Delta$ is defined in Assumption~\ref{ass:omega}.
    
The term in line \eqref{eq:sum_ll_err} may be written as follows
    \begin{align*}
        &\hat{p}(y_t|y_{1:t-1}) - p(y_t|y_{1:t-1}) \\
        =& \iint g(y_t|x_t)f(\mathrm{d}x_t|\tilde{x}_{t-1})\tilde{\beta}^{(t-1)}_N(\mathrm{d}\tilde{x}_{t-1})-\iint g(y_t|x_t)f(\mathrm{d}x_t|\tilde{x}_{t-1})\tilde{\beta}^{(t-1)}(\mathrm{d}\tilde{x}_{t-1})\\
    =&\tilde{\beta}^{(t-1)}_N(h) - \beta^{(t-1)}(h)
    \end{align*}
                    for $\Delta^2 \leq h(x): = \int g(y_t|x')f(x'|x)\mathrm{d}x'\leq \Delta^{-2}$.
                    At this point notice also that 
    \begin{align*}
    h(x) - h(x')
    &= \int f(\mathrm{d} w|x) g(y_t \mid w) - \int f(\mathrm{d} w|x') g(y_t \mid w)\\
&= \int \delta_x(\mathrm{d} z) \int f(\mathrm{d} w |z) g(y_t \mid w) - \int \delta_{x'}(\mathrm{d} z) \int f(\mathrm{d} w |z) g(y_t \mid w)\\
&= [\delta_x f][g(y_t\mid \cdot)]-[\delta_{x'}f][g(y_t\mid \cdot)]\\
&\leq \mathrm{Lip}\left[g(y_t \mid \cdot) \right] \wass_1\left(\delta_x f, \delta_{x'} f\right)
\leq \kappa \mathrm{Lip}\left[g(y_t \mid \cdot )\right] \wass_1\left(\delta_x , \delta_{x'} \right)
=\kappa\mathrm{Lip}\left[g(y_t \mid \cdot )\right] |x-x'|, 
    \end{align*}
by     Assumption~\ref{ass:lipcontraction}. It follows therefore that $h$ is Lipschitz and therefore that 
\begin{align*}
    \hat{p}(y_t|y_{1:t-1}) - p(y_t|y_{1:t-1}) 
        &=\tilde{\beta}^{(t-1)}_N(h) - \beta^{(t-1)}(h)
        \leq \kappa \mathrm{Lip}\left[ g(y_t \mid \cdot) \right] \wass_1\left(\beta^{(t-1)}_N, \beta^{(t-1)}\right).
    \end{align*}
Combining \eqref{eq:timetbound} and \eqref{eq:sum_ll_err}, and using the fact that $\wass_1\leq \wass_2$, we thus get
\begin{align*}
    \big|\log \hat{p}(y_{1:T}) - \log p(y_{1:T})\big| 
    &\leq \Delta^{-1} \kappa\sum_{t=1}^T \mathrm{Lip}\left[g(y_t\mid \cdot) \right] \wass_1\left(\tilde{\beta}_N^{(t-1)}, \beta^{(t)} \right)\\
    &\leq \Delta^{-1} \kappa \max_{t\in [1:T]}\mathrm{Lip}\left[g(y_t\mid \cdot) \right]
    \sum_{t=1}^T  \mathfrak{G}_{\epsilon, \delta/T, N, d}^{(t)}\left(\sqrt{f_d^{-1} \left(\frac{\log(CT/\delta)}{cN} \right)}\right),
\end{align*}
where the last inequality holds with probability at least $1-\delta$ over the sampling steps. 

The convergence in probability follows from the corresponding statement of Proposition~\ref{prop:recursivebound}.
\end{proof}

\section{Additional Experiments and Details}
    \subsection{Linear Gaussian model}
        We first consider the following 2-dimensional linear Gaussian SSM for which exact inference can be carried out using Kalman techniques:
        \begin{equation}\label{eq:exampleLGSS}
            X_t|\{ X_{t-1}=x \}\sim \mathcal{N} \left(\text{diag}(\theta_1~\theta_2) x, 0.5 \mathbf{I}_2\right), \quad Y_t|\{ X_t=x\}\sim \mathcal{N}(x, 0.1 \cdot \mathbf{I}_2).
        \end{equation}
        We simulate $T=150$ observations using $\theta=(\theta_1,\theta_2)=(0.5,0.5)$. As a result, we expect in these scenarios that the filtering distribution $p_{\theta}(x_t|y_{1:t})$ is not too distinct from the smoothing distribution  $p_{\theta}(x_t|y_{1:T})$ as the latent process is mixing quickly. From Proposition \ref{prop:asymptoticELBOgradient}, this is thus a favourable scenario for methods ignoring resampling terms in the gradient as the bias should not be very large.  Figure \ref{fig:surfaces}, displayed earlier, shows $\ell(\theta)$  obtained by Kalman and $\hat{\ell}(\theta;\bf{u})$ computed regular PF and DPF for the same number $N=25$ of particles using $q_{\phi}(x_t|x_{t-1},y_t)=f_{\theta}(x_t|x_{t-1})$. The corresponding gradient vector fields are given in Figure \ref{fig:vectorField}, where the gradient is computed using the biased gradient from \cite{maddison2017filtering,naesseth2017variational,le2017auto} for regular PF. 
        
        
        We now compare the performance of the estimators $\hat{\theta}_{\text{SMLE}}$ (for DPF) and  $\hat{\theta}_{\text{ELBO}}$ (for both regular PF and DPF) learned using gradient with learning rate $10^{-4}$ on 100 steps, using $N=25$ for DPF and $N=500$ for regular PF, to $\hat{\theta}_{\text{MLE}}$ computed using Kalman derivatives. We simulate $M=50$ realizations of $T=150$ observations using $\theta=(\theta_1,\theta_2)=(0.5,0.5)$. The ELBO stochastic gradient estimates are computed using biased gradient estimates of $\ell_{\text{ELBO}}(\theta)$ ignoring the contributions of resampling steps as in \cite{maddison2017filtering,naesseth2017variational,le2017auto} (we recall that unbiased estimates suffer from very high variance) and unbiased gradients of $\ell^{\text{ELBO}}(\theta)$ using DPF.  We average $B$ parallel PFs to reduce the variance of these gradients of the ELBO and also $B$ PFs (with fixed random seeds) to compute the gradient of $\hat{\ell}_{\text{SMLE}(\theta;\mathbf{u}_{1:B})}:=\tfrac 1 B \sum_{b=1}^B \hat{\ell}(\theta;\mathbf{u}_b)$. The results are given in Table \ref{tab:theta_diff}. For this example, $\hat{\theta}^{\text{DPF}}_{\text{ELBO}}$ maximizing $\ell^{\text{ELBO}}_{\textrm{DPF}}(\theta)$ outperforms $\hat{\theta}^{\textrm{PF}}_{\text{ELBO}}$ and $\hat{\theta}_{\text{SMLE}}$. However, as $B$ increases, $\hat{\theta}_{\text{SMLE}}$ gets closer to $\hat{\theta}^{\text{DPF}}_{\text{ELBO}}$ which is to be expected as $\hat{\ell}_{\text{SMLE}}(\theta;\mathbf{u}_{1:B}) \longrightarrow \ell^{\text{ELBO}}(\theta)$. In Table \ref{tab:theta_diff}, the Root Mean Square Error (RMSE) is defined as $\sqrt{\tfrac 1 M \sum_{i=1}^2\sum_{k=1}^M(\hat{\theta}^k_i - \hat{\theta}^k_{\text{MLE},i})^2}$.

        \begin{table}[H]
        \centering
        \captionsetup{justification=centering}
        \caption{$10^3\times$ RMSE\footnotemark~ over 50 datasets - lower is better}


\begin{tabular}{llll}
\toprule
$B$ &  $\hat{\theta}^\text{PF}_{\text{ELBO}}$ & $\hat{\theta}^\text{DPF}_{\text{ELBO}}$ & $\hat{\theta}_{\text{SMLE}}$ \\ \midrule
1         & 1.94                 & 1.30                 & 7.94                 \\
4          & 2.40                 & 1.35                 & 3.28                  \\
10         & 2.80                 & 1.37                & 2.18                  \\
\bottomrule
\end{tabular}
        \label{tab:theta_diff}
        \end{table}
        
        \footnotetext{The Root Mean Square Error (RMSE) is defined as $\sqrt{\tfrac 1 M \sum_{i=1}^2\sum_{k=1}^M(\hat{\theta}^k_i - \hat{\theta}^k_{\text{MLE},i})^2}$.}
        

        \subsection{Variational Recurrent Neural Network}
            $N=32$ particles were used for training, with a regularization parameter of $\epsilon=0.5$. The ELBO (scaled by sequence length) was used as the training objective to maximise for each resampling/ DET procedure. The ELBO evaluated on test data using $N=500$ particles and multinomial resampling. Resampling / DET operations were carried out when effective sample (ESS) size fell below $N/2$. Learning rate $0.001$ was used with the Adam optimizer.
            
            Recall the state-space model is given by
            \begin{align*}
            (R_{t}, O_{t}) &= \text{RNN}_\theta(R_{t-1}, Y_{1:t-1}, E_\theta(Z_{t-1})),\\
            Z_{t} &\sim \mathcal{N}(\mu_{\theta}(O_{t}),
             \sigma_{\theta}(O_{t})), \\
            \hat{p}_t &= h_\theta(E_\theta(Z_t), O_{t}), \\ 
             Y_t | X_t  &\sim \text{Ber}(\hat{p}_t).
        \end{align*}
        
        Network architectures and data preprocessing steps were based loosely on \cite{maddison2017filtering}. Given the low volume of data and sparsity of the observations, relatively small neural networks were considered to prevent overfitting, larger neural networks are considered in the more complex robotics experiments.
        $R_t$ is of dimension $d_r=16$, $Z_t$ is of dimension $d_z=8$. $E_\theta$ is a single layer fully connected network with hidden layer of width $16$, output of dimension $16$ and RELU activation.
        
        $\mu_\theta$ and $\sigma_\theta$ are both fully connected neural networks with two hidden layers, each of $16$ units and RELU activation, the activation function is not applied to the final output of $\mu_\theta$ but the softplus is applied to the output of $\sigma_\theta$, which is the diagonal entries of the covariance matrix of the normal distribution that is used to sample $Z_t$.
        
        $h_\theta$ is a single layer fully connected network with two hidden layers, each of width $16$ and RELU activation.  The final output is not put through the RELU and is instead used as the logits for the Bernoulli distribution of observations. 
        
        \subsection{Robot Localization}
        Similar to the VRNN example, $N=32$ particles were used for training, with a regularization parameter of $\epsilon=0.5$ and resampling / DET operations were carried out when ESS size fell below $N/2$. Learning rate $0.001$ was used with the Adam optimizer.
        
        Network architectures and data preprocessing were based loosely on \cite{jonschkowski2018differentiable}. There are $3$ neural networks being considered:
        \begin{itemize}
            \item Encoder $E_\theta$ maps RBG $24\times 24$ pixel images, hence dimension $3\times 24 \times 24$, to encoding of size $d_E = 128$. This network consists of a convolutional network (CNN) of kernel size $3$ and a single layer fully connected network of hidden width $128$ and RELU activation. 
            \item Decoder $D_\theta$ maps encoding back to original image. This consists of a fully connected neural network with three hidden layers of width $128$ and RELU activation function. This is followed by a transposed convolution network with matching specification to the CNN in the encoder, to return an output with the same dimension as observation images, $3\times 24 \times 24$.
            \item Network $G_\theta$ maps the state $S_t=(X^{(1)}_t, X^{(2)}_t, \gamma_t)$ to encoding of dimension $128$. First angle $\gamma_t$ was converted to $\sin(\gamma_t), \cos(\gamma_t)$. Then the augmented state $(X^{(1)}_t, X^{(2)}_t, \sin(\gamma_t), \cos(\gamma_t))$ was passed to a $3$ layer fully connected network with hidden layers of dimensions $16,32,64$ and RELU activation function, with final output of dimension $128$. 
        \end{itemize}
\end{document}